\documentclass[twoside,11pt]{article}

\usepackage{jmlr2e}

\usepackage{microtype}
\usepackage{graphicx}
\usepackage{booktabs}
\usepackage{subcaption}
\usepackage{algorithm}
\usepackage{algorithmic}
\newcommand{\R}{\mathbb{R}}
\newcommand{\E}{\mathbb{E}}
\newcommand{\loss}{\mathcal{L}}
\newcommand{\ham}{\mathcal{H}}   
\newcommand{\Ncal}{\mathcal{N}}
\newcommand{\Scal}{\mathcal{S}}

\newcommand{\Lie}{\mathcal{L}}
\newcommand{\ip}[2]{\left\langle #1,#2\right\rangle}
\newcommand{\norm}[1]{\left\lVert #1\right\rVert}

\newcommand{\Span}{\operatorname{span}}
\newcommand{\Ker}{\operatorname{Ker}}

\newcommand{\Ran}{\operatorname{Ran}}
\newcommand{\pop}{\mathrm{pop}}
\newcommand{\grok}{\mathrm{g}}
\newcommand{\eq}{\mathrm{eq}}
\newcommand{\euc}{\mathfrak{euc}}
\newcommand{\PN}{\Pi_{N}}
\newcommand{\Pperp}{\Pi_{\perp}}
\newcommand{\Pg}{\Pi_{\mathrm{g}}}
\newcommand{\Pzero}{\Pi_{0}}
\newcommand{\so}{\mathfrak{so}}

\ShortHeadings{Grokking on the Weight-Decay Clock}{Kim}
\firstpageno{1}

\begin{document}

\title{Grokking on the Weight-Decay Clock:\\
A Rate Hierarchy from Softly Broken Symmetries}

\author{\name Taeyoung Kim \email taeyoungkim@kias.re.kr\\
       \addr Center for AI and Natural Sciences\\
       Korea Institute for Advanced Study\\
       Seoul 02455, Republic of Korea}

\maketitle

\begin{abstract}
Delayed generalization, or grokking, remains poorly understood despite extensive empirical study. We identify an exactly solvable late-time relaxation mechanism for grokking in linear models trained with full-batch heavy-ball optimization and weight decay, together with a locally quadratic extension to nonlinear neural networks. Our analysis reveals a distinguished population-active component of the empirical null space, which we call the grokking subspace. Along this subspace, the training predictions remain unchanged, leaving weight decay as the sole restoring force and giving rise to a slow dissipative relaxation governed by exact discrete-time and continuous-time laws. We show that only this subspace contributes to the slow asymptotic decay of the population risk and derive explicit iteration-scale predictions for the grokking time, recovering the familiar $(1-\beta)/(\eta\lambda)$ scaling in the weak-regularization regime. The theory further predicts distinct effects of optimizer choice, distinguishing coupled $L_2$ regularization from decoupled weight decay, and yields causal predictions for interventions that modify the grokking component. We verify all theoretical identities without fitted parameters in a synthetic model where every subspace and relaxation rate is computable in closed form. We further observe delayed generalization in modular addition, where the measured delay follows the predicted scaling and the late-time relaxation agrees closely with the theoretical clock. 
\end{abstract}


\section{Introduction}
\label{sec:introduction}

Neural networks often improve their test performance only well after the
training loss has converged \citep{power2022grokking}; this delayed
generalization is termed \emph{grokking}. The delay is highly sensitive
to weight decay, learning rate, momentum, and dataset size, which points
to an underlying dynamical cause rather than a purely representational
one.

Existing work can be broadly grouped into norm-based, mechanistic, and
dynamical explanations. Norm-based accounts argue that optimization
first reaches a high-norm memorizing solution, after which
regularization selects a lower-complexity solution
\citep{liu2022omnigrok,junior2025grokking}. Mechanistic and
feature-learning studies identify the delayed formation of structured
representations, such as Fourier-like circuits in modular arithmetic
\citep{nanda2023progress,gromov2023grokking,mohamadi2024why}. Dynamical
analyses show that, under small weight decay, optimization rapidly
approaches an interpolation manifold and subsequently follows a slow
norm-minimizing flow on a timescale of order $1/\lambda$
\citep{boursier2025theoretical,xu2026grok,truong2026norm}.

The slow $1/\lambda$ timescale itself is therefore not new. What the
existing accounts leave open is which part of the slow motion
matters for generalization, and how the answer depends on the
optimizer. Previous work identifies slow motion along the interpolation
manifold; we further decompose its tangent space into
population-invisible directions and population-active empirical-null
directions, and show that only the latter contribute to the late-time
population-risk tail. Concretely, the empirical Gram matrix $G_N$ and
its population counterpart $G_\pop$ have nested null spaces
$\Ker G_\pop\subseteq\Ker G_N$, and the Euclidean representative of the
quotient,
\[
    \Scal
    =
    \Ker G_N\cap(\Ker G_\pop)^\perp ,
\]
is the subspace along which the network predictions are frozen on the
training set yet still move at the population level. We call $\Scal$
the \emph{grokking subspace}.

We begin with the exactly solvable case and treat the continuous-time, Hamiltonian, and nonlinear settings as perturbations of it. For models that are linear in their parameters, trained with
the squared loss and heavy-ball momentum with coupled $L_2$
regularization, the full mode spectrum of the training recurrence is
available in closed form, with no continuous-time or overdamped
approximation. Translations along $\Ker G_N$ are exact symmetries of
the training loss; weight decay breaks them softly, and the resulting
soft Noether law is literally the equation of motion of the slow
coordinate. The continuous-time picture, the Hamiltonian formulation,
and the local analysis of nonlinear networks are then controlled
perturbations of this exact core.

Our contributions are:
\begin{itemize}
    \item a characterization of post-interpolation motion through the
    empirical--population kernel mismatch: the tangent space of the
    interpolation manifold splits as
    $\Ker G_N=\Scal\oplus\Ker G_\pop$, and only the population-active
    component $\Scal$ contributes to the slow empirical-null part of
    the asymptotic population-risk tail --- transverse directions
    remain in a faster remainder
    (Theorem~\ref{thm:population-tail});
    \item the exact discrete mode spectrum of heavy-ball training with
    coupled $L_2$ regularization, including the exact slow root
    $\rho_0$ and the iteration-scale law
    $k_{\mathrm{grok}}\approx\frac{1-\beta}{\eta\lambda}
    \log(\norm{\Pg\theta_0}/\varepsilon)$, with
    $\eta\lambda/(1-\beta)$ and $\lambda/\gamma$ exhibited as
    weak-regularization asymptotics of exact roots
    (Theorem~\ref{thm:discrete-spectrum},
    Corollary~\ref{cor:iteration-clock});
    \item a soft Noether law centered on the translation symmetries of the empirical null space: the translation charges are
    the momenta of the slow coordinates, their softly broken
    conservation law is the slow-mode ODE, and the rotation charges
    supply a faster angular clock with an exact discrete counterpart
    $Q_k=\beta^{\,k-1}Q_1$ (Theorem~\ref{thm:soft-noether},
    Corollary~\ref{cor:rotation-charges});
    \item falsifiable predictions are checked at two levels (Section~\ref{sec:experiments}). First, an exactly computable synthetic model verifies, without fitted parameters, the $(1-\beta)/(\eta\lambda)$ collapse of the relaxation count, the logarithmic delay shift produced by interventions that rescale the grokking component while preserving all training predictions, the train/test signature distinguishing $\Scal$ from genuine population-null directions, and the factor-$(1-\beta)$ difference between coupled $L_2$ and decoupled weight decay. Second, a modular-addition benchmark exhibits genuine delayed generalization with a delay that follows the predicted $(1-\beta)/(\eta\lambda)$ scaling (measured log–log slope $1.01$).
\end{itemize}

\paragraph{Scope.}
Our theorems concern squared loss (or locally quadratic regimes),
full-batch heavy-ball dynamics, and weak regularization
$\eta\lambda<(1-\sqrt\beta)^2$; nonlinear networks are treated locally,
with explicit remainder assumptions and an explicit error floor
(Section~\ref{sec:nonlinear}). Equally important is what the mechanism
does \emph{not} explain: grokking without explicit weight decay; the
creation of new population-visible directions by feature learning (our
projectors are frozen at a reference point, and the benchmark of
Section~\ref{sec:modular} shows this matters quantitatively); the
selection or diffusion effects of mini-batch noise; the abrupt jump of
test accuracy, which requires a margin argument beyond the loss
tail; quantitative dataset-size laws ($N$ enters through subspace
dimensions, initial amplitudes, and the compatibility condition, not
through the rate $\rho_0$ itself); and adaptive or preconditioned
optimizers, deferred to a brief outlook
(Appendix~\ref{app:natural-gradient}). What the paper does provide is
one computable late-time mechanism with causal signatures that
are testable beyond the solvable regime.

\section{Setup: Heavy-Ball Dynamics and Kernel Geometry}
\label{sec:setup}

\subsection{Model, Loss, and Regularization}
\label{sec:model}

Let $f(\cdot;\theta):\R^{n_0}\to\R$ be a model with parameters
$\theta\in\Theta\simeq\R^P$, trained on a dataset
$\{(x_i,y_i)\}_{i=1}^N$ drawn from a population distribution $\mu$ over
inputs with target function $y_\star\in L^2_\mu$. Throughout the main
text we use the \emph{squared loss}; the empirical and population
risks and the regularized objective are
\begin{equation}
    \loss_N(\theta)
    =
    \frac{1}{2N}\sum_{i=1}^N
    \bigl(f(x_i;\theta)-y_i\bigr)^2,
    \qquad
    \loss_\pop(\theta)
    =
    \frac12\,
    \E_{x\sim\mu}\Big[\bigl(f(x;\theta)-y_\star(x)\bigr)^2\Big],
\label{eq:losses}
\end{equation}
\begin{equation}
    U(\theta)
    =
    \loss_N(\theta)
    +
    \frac{\lambda}{2}\norm{\theta}^2 .
\label{eq:objective}
\end{equation}
Scalar outputs keep the notation light; all statements extend to
$\R^{n_L}$-valued outputs by stacking.

\begin{remark}[General losses]
\label{rem:general-loss}
For a general twice-differentiable loss $\ell(f,y)$ the natural
substitute for the Gram matrices below is the generalized Gauss--Newton
metric
$G_N^{\ell}=\frac1N\sum_i \partial_\theta f_i^\top\,
\partial_f^2\ell_i\,\partial_\theta f_i$. We restrict the main text to
the squared loss, where $G_N^{\ell}=G_N$ and the theory is free
of curvature-residual terms; the general case is discussed in
Section~\ref{sec:nonlinear}.
\end{remark}

\subsection{Heavy-Ball Recurrence and Its Continuous Limit}
\label{sec:heavy-ball}

We analyze the heavy-ball scheme with \emph{coupled} $L_2$
regularization,
\begin{equation}
    v_{k+1}
    =
    \beta v_k-\eta\nabla U(\theta_k),
    \qquad
    \theta_{k+1}
    =
    \theta_k+v_{k+1},
\label{eq:heavy-ball}
\end{equation}
with learning rate $\eta>0$ and momentum $\beta\in[0,1)$.

\begin{remark}[Coupled $L_2$ versus decoupled weight decay]
\label{rem:coupled-decoupled}
In \cref{eq:heavy-ball} the regularizer enters through $\nabla U$ and
is therefore integrated by the momentum buffer. The decoupled
variant (SGDW-style) applies the decay directly to the iterate,
$\theta_{k+1}=\theta_k+v_{k+1}-\eta\lambda\theta_k$ with
$v_{k+1}=\beta v_k-\eta\nabla\loss_N(\theta_k)$. For plain gradient
descent the two coincide; with momentum they induce different
null-space dynamics and different clocks
(Proposition~\ref{prop:decoupled}). The two update rules as
displayed here are the objects analyzed in this paper; deep-learning
frameworks implement further variants (decay applied before or after
the buffer update, learning-rate--coupled decay, schedules), so every
coupled-versus-decoupled statement below is to be read ``under this
convention.'' Adaptive optimizers introduce further differences that
we do not analyze.
\end{remark}

Eliminating $v_k=\theta_k-\theta_{k-1}$ turns \cref{eq:heavy-ball}
into the second-order difference equation
\begin{equation}
    \theta_{k+1}-(1+\beta)\theta_k+\beta\theta_{k-1}
    =
    -\eta\nabla U(\theta_k).
\label{eq:second-order-difference}
\end{equation}
Introducing the interpolation $\theta_k=\theta(t_k)$, $t_k=k\delta$,
and Taylor-expanding, we find that the diffusive time step
$\delta=\sqrt\eta$ balances the inertial and friction terms and yields,
to leading order in $\sqrt\eta$,
\begin{equation}
    m\ddot{\theta}
    +
    \gamma\dot{\theta}
    =
    -\nabla U(\theta),
    \qquad
    m=\frac{1+\beta}{2},
    \qquad
    \gamma=\frac{1-\beta}{\sqrt{\eta}} .
\label{eq:damped-particle}
\end{equation}
The status of \cref{eq:damped-particle} needs care. With
$\delta=\sqrt\eta$ the recurrence admits an $\eta$-dependent
modified equation with the stated coefficients; it is not a limit at
fixed $\beta$, since $\gamma=(1-\beta)/\sqrt\eta$ diverges as
$\eta\to0$ with $\beta$ fixed, and the dynamics then becomes strongly
overdamped. A finite-friction continuous \emph{limit} additionally
requires the joint scaling $1-\beta=\Gamma\sqrt\eta+o(\sqrt\eta)$, in
which case $\gamma\to\Gamma$ and $m\to1$. We therefore treat
\cref{eq:damped-particle} throughout as a modified-equation companion
at the given $(\eta,\beta)$, valid to leading order in $\sqrt\eta$,
with iteration counts converting through $t=k\sqrt\eta$: every
headline statement is proved directly for the recurrence
\cref{eq:second-order-difference}, and continuous-time versions are
stated alongside with their own exact rates. The regime $\beta\to1$
corresponds to weakly damped dynamics.

\subsection{Observability Gram Matrices and the Grokking Subspace}
\label{sec:kernel-geometry}

Let $J_x(\theta):=\partial f(x;\theta)/\partial\theta\in\R^{1\times P}$
and define the empirical and population Gram matrices
\begin{equation}
    G_N(\theta)
    =
    \frac{1}{N}\sum_{i=1}^N
    J_{x_i}(\theta)^\top J_{x_i}(\theta),
    \qquad
    G_\pop(\theta)
    =
    \int
    J_x(\theta)^\top J_x(\theta)
    \,\mathrm{d}\mu(x).
\label{eq:gram-matrices}
\end{equation}
We emphasize the intended role of these matrices: they are
\emph{observability Grams}, quantifying how visible an infinitesimal
parameter motion $v$ is in prediction space,
\[
    v^\top G_N v
    =
    \frac1N\sum_i\bigl(J_{x_i}v\bigr)^2,
    \qquad
    v^\top G_\pop v
    =
    \bigl\|J_\cdot\,v\bigr\|_{L^2_\mu}^2 .
\]
These Gram matrices quantify observability in prediction space rather than inertia in parameter space. In the Hamiltonian formulation, the kinetic metric is instead the optimizer-induced scalar metric $mI$ (Section~\ref{sec:symmetries}); metric-preconditioned extensions are deferred to Appendices~\ref{app:noether} and~\ref{app:natural-gradient}.

Define the null spaces (at a fixed reference $\theta$, suppressed from
the notation when clear)
\[
    \Ncal_N
    =
    \Ker G_N
    =
    \{v:\ J_{x_i}v=0\ \text{for all }i\},
    \qquad
    \Ncal_\pop
    =
    \Ker G_\pop
    =
    \{v:\ J_xv=0\ \ \mu\text{-a.e.}\},
\]
and write $r_N:=\dim\Ncal_N$, $r_\pop:=\dim\Ncal_\pop$ for their
dimensions. (The nullity $r_N$ is distinct from the input dimension
$n_0$.)

\begin{lemma}[Kernel inclusion under sampling]
\label{lem:kernel-inclusion}
For every dataset, $\Ncal_\pop\subseteq\Ncal_N$ fails only if some
training input lies in one of the exceptional $\mu$-null sets
$\{x:J_xv\neq0\}$, $v\in\Ncal_\pop$. In particular, if
$x_1,\dots,x_N$ are drawn i.i.d.\ from $\mu$, then
\[
    \Ncal_\pop\subseteq\Ncal_N
    \qquad\text{with probability one.}
\]
\end{lemma}

\begin{proof}
If $v\in\Ncal_\pop$ then $v^\top G_\pop v=\norm{J_\cdot v}_{L^2_\mu}^2=0$,
so $J_xv=0$ for $\mu$-a.e.\ $x$. For i.i.d.\ samples, each fixed basis
vector $v$ of $\Ncal_\pop$ satisfies $J_{x_i}v=0$ for all $i$ almost
surely; intersecting the finitely many probability-one events over a
basis gives the claim. The deterministic statement is the contrapositive.
\end{proof}

The lemma is a statement at a fixed reference $\theta$ (for
the linear model the Jacobians are constant, so this is immaterial;
for nonlinear models the null spaces are $\theta$-dependent). We do not claim that the inclusion holds simultaneously along a
data-dependent training trajectory $\theta_k$, which would raise an
adaptivity issue; the local analysis of Section~\ref{sec:nonlinear}
therefore works exclusively with projectors frozen at a single
reference point.

The inclusion is generally strict, and the \emph{intrinsic} object of
interest is the quotient $\Ncal_N/\Ncal_\pop$: the space of tangent
classes that are invisible on the training set but visible at the
population level. The quotient is defined without reference to any
inner product and is therefore reparameterization-natural. To compute
with it we choose the orthogonal representative induced by the
optimizer's kinetic metric $mI$ (Section~\ref{sec:symmetries}),
\begin{equation}
    \Scal
    :=
    \Ncal_N\cap\Ncal_\pop^{\perp},
    \qquad
    r_{\grok}:=\dim\Scal=r_N-r_\pop ,
\label{eq:grokking-subspace}
\end{equation}
and we emphasize that $\Scal$ and its projector are
metric-dependent choices: under a general constant kinetic
metric $M$ the natural representative is
$\Scal_M=\Ncal_N\cap\Ncal_\pop^{\perp_M}$, and under a
reparameterization the orthogonal complement changes accordingly. All
statements about ``the grokking subspace'' below are statements about
the pair (quotient, chosen metric); we do not claim that $\Scal$
itself is reparameterization-invariant. With this understood, the
parameter space splits orthogonally as
\begin{equation}
    \R^P
    =
    \underbrace{\Ncal_N^{\perp}}_{\text{fit directions}}
    \ \oplus\
    \underbrace{\Scal}_{\text{grokking subspace}}
    \ \oplus\
    \underbrace{\Ncal_\pop}_{\text{population-null}} ,
    \qquad
    \Ncal_N=\Scal\oplus\Ncal_\pop .
\label{eq:decomposition}
\end{equation}
We write $\PN$, $\Pperp=I-\PN$, $\Pg$, and $\Pzero$ for the
orthogonal projections onto $\Ncal_N$, $\Ncal_N^\perp$, $\Scal$, and
$\Ncal_\pop$. Directions in $\Scal$ leave every training prediction
unchanged to first order yet change the function on a set of positive
$\mu$-measure; in the chosen metric, they are the carriers of the
post-interpolation motion that can change generalization.
Figure~\ref{fig:decomposition} summarizes the decomposition.

\begin{figure}[t]
\centering
\includegraphics[width=0.9\textwidth]{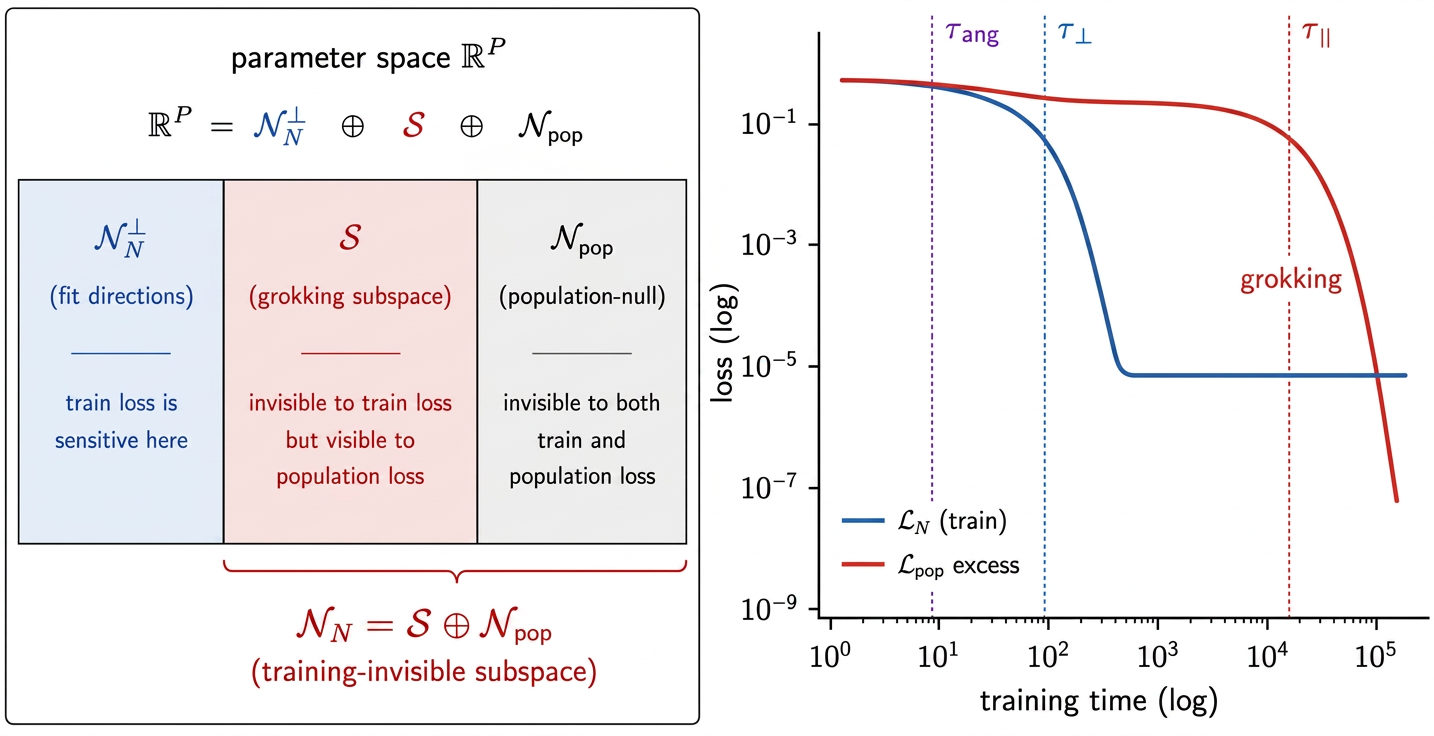}
\caption{Kernel decomposition and the three clocks. Left: the
parameter space splits into fit directions $\Ncal_N^{\perp}$, the
grokking subspace $\Scal=\Ncal_N\cap\Ncal_\pop^\perp$, and the
population-null space $\Ncal_\pop$; the training loss sees only
$\Ncal_N^\perp$, the population loss sees $\Ncal_N^\perp\oplus\Scal$.
Right: schematic trajectory: fast fit at rates $(h_j+\lambda)/\gamma$,
angular relaxation at $\gamma/m$, and a slow weight-decay contraction
of the grokking component at
$r_{\mathrm{slow}}\approx\lambda/\gamma$, which alone sets the slow
component of the late-time population-risk tail.}
\label{fig:decomposition}
\end{figure}

\subsection{Reference Points: Interpolators, Minimum-Norm Solutions, and Population Stationarity}
\label{sec:reference-points}

Three conceptually different reference points appear in analyses of
grokking, and much confusion follows from conflating them:
\begin{align*}
    \theta_{\mathrm{int}} &:\ \text{an interpolator reached by the fast
    dynamics, } \loss_N(\theta_{\mathrm{int}})=0;\\
    \theta_{\min} &:\ \text{the minimum-$\ell_2$-norm interpolator,
    characterized by } \PN\theta_{\min}=0;\\
    \theta_\pop &:\ \text{a stationary point of the population risk, }
    \nabla\loss_\pop(\theta_\pop)=0 .
\end{align*}
These coincide only under additional hypotheses. Our results are
organized so that each statement names the reference point it actually
needs: the exact linear theory of
Sections~\ref{sec:exact-linear}--\ref{sec:population-tail} is written
around the regularized equilibrium
$\theta_\eq:=\arg\min_\theta U(\theta)$, which exists and is unique for
$\lambda>0$, satisfies $\PN\theta_\eq=0$, and converges to
$\theta_{\min}$ as $\lambda\to0$; population stationarity is
never assumed implicitly, but enters
Theorem~\ref{thm:population-tail} through an explicit compatibility
condition on $\Pg\nabla\loss_\pop(\theta_\eq)$.

\section{Exact Theory for the Linear Model}
\label{sec:exact-linear}

This section is the exact core of the paper. We take the model to be
linear in its parameters,
\begin{equation}
    f(x;\theta)=\phi(x)^\top\theta,
    \qquad
    \phi:\R^{n_0}\to\R^P\ \text{fixed},
\label{eq:linear-model}
\end{equation}
so that $J_x=\phi(x)^\top$ is independent of $\theta$, and all Gram
matrices are constant. Write $\phi_i:=\phi(x_i)$,
$\Phi\in\R^{N\times P}$ for the matrix with rows $\phi_i^\top$,
$\mathsf J_N:=\Phi$,
\[
    G_N=\frac1N\Phi^\top\Phi,
    \qquad
    g_N=\frac1N\Phi^\top y,
    \qquad
    G_\pop=\E_\mu\bigl[\phi(x)\phi(x)^\top\bigr],
    \qquad
    g_\pop=\E_\mu\bigl[y_\star(x)\,\phi(x)\bigr].
\]
Both $\loss_N$ and $\loss_\pop$ are then quadratic:
$\loss_N(\theta)=\frac12\theta^\top G_N\theta-g_N^\top\theta+c_N$ and
$\loss_\pop(\theta)=\frac12\theta^\top G_\pop\theta
-g_\pop^\top\theta+c_\pop$. Since $g_N\in\Ran(G_N)=\Ncal_N^\perp$, the
regularized objective $U$ has the unique minimizer
\begin{equation}
    \theta_\eq=(G_N+\lambda I)^{-1}g_N\in\Ncal_N^{\perp},
    \qquad\text{so}\qquad
    \PN\theta_\eq=0 .
\label{eq:ridge-equilibrium}
\end{equation}
The identity $\PN\theta_\eq=0$ is a property of the equilibrium, not an
extra assumption; it is what later removes all weight-decay cross terms
without any appeal to ``centered coordinates''.

\subsection{Exact Discrete Mode Spectrum}
\label{sec:discrete-spectrum}

Let $\xi_k:=\theta_k-\theta_\eq$. Substituting
$\nabla U(\theta)=(G_N+\lambda I)(\theta-\theta_\eq)$ into
\cref{eq:second-order-difference} gives the exact linear recurrence
\begin{equation}
    \xi_{k+1}
    =
    \bigl[(1+\beta)I-\eta(G_N+\lambda I)\bigr]\xi_k
    -\beta\,\xi_{k-1} .
\label{eq:exact-recurrence}
\end{equation}
Let $\{(h_j,u_j)\}_{j=1}^{P}$ be an orthonormal eigensystem of $G_N$,
with $h_j>0$ for $u_j\in\Ncal_N^\perp$ and $h_j=0$ for
$u_j\in\Ncal_N$. In the coordinate $q_{j,k}:=u_j^\top\xi_k$ the
recurrence decouples into scalar three-term recursions
\begin{equation}
    q_{j,k+1}
    =
    a_j\,q_{j,k}-\beta\,q_{j,k-1},
    \qquad
    a_j:=1+\beta-\eta(h_j+\lambda),
\label{eq:modal-recurrence}
\end{equation}
with characteristic roots
\begin{equation}
    \rho_j^{\pm}
    =
    \frac{a_j\pm\sqrt{a_j^2-4\beta}}{2} .
\label{eq:characteristic-roots}
\end{equation}
On the empirical null space the recursion is \emph{data-independent}:
every coordinate $q=\ip{\xi}{b}$ with $b\in\Ncal_N$ obeys
\begin{equation}
    q_{k+1}=(1+\beta-\eta\lambda)\,q_k-\beta\,q_{k-1} .
\label{eq:null-recurrence}
\end{equation}

\begin{theorem}[Exact discrete mode spectrum and rate hierarchy]
\label{thm:discrete-spectrum}
Consider the linear model \cref{eq:linear-model} with squared loss,
trained by the heavy-ball iteration \cref{eq:heavy-ball}. Then every
trajectory is an exact superposition of modes,
\[
    \xi_k=\sum_j\bigl(c_j^+(\rho_j^+)^k+c_j^-(\rho_j^-)^k\bigr)u_j
\]
(with the usual confluent modification $(c_j^0+c_j^1k)\rho^k$ when
$a_j^2=4\beta$), and:
\begin{enumerate}
\item[(i)] \textbf{Stability.} $\max(|\rho_j^+|,|\rho_j^-|)<1$ for all
$j$ if and only if $0<\eta(h_j+\lambda)<2(1+\beta)$ for all $j$.
\item[(ii)] \textbf{Null modes.} If
\begin{equation}
    \eta\lambda<(1-\sqrt\beta)^2 ,
\label{eq:overdamped-null}
\end{equation}
then for $h_j=0$ both roots are real and positive,
\begin{equation}
    \rho_0
    :=
    \frac{a_0+\sqrt{a_0^2-4\beta}}{2},
    \qquad
    \rho_1:=\frac{\beta}{\rho_0},
    \qquad
    a_0=1+\beta-\eta\lambda ,
\label{eq:slow-root}
\end{equation}
with $\beta\le\rho_1<\sqrt\beta<\rho_0<1$, and the slow rate admits the
expansion
\begin{equation}
    1-\rho_0
    =
    \frac{\eta\lambda}{1-\beta}
    +O\!\left(\frac{(\eta\lambda)^2}{(1-\beta)^3}\right).
\label{eq:slow-root-expansion}
\end{equation}
\item[(iii)] \textbf{Transverse modes.} For $h_j>0$ with
$(1-\sqrt\beta)^2<\eta(h_j+\lambda)<(1+\sqrt\beta)^2$ the roots form a
complex pair with $|\rho_j^\pm|=\sqrt\beta$; for
$\eta(h_j+\lambda)\le(1-\sqrt\beta)^2$ they are real with
$\rho_j^+<\rho_0$. Consequently, if in addition to
\cref{eq:overdamped-null} the no-sign-flip condition
$\eta(h_j+\lambda)\le(1+\sqrt\beta)^2$ holds for all $j$, then
\begin{equation}
    \max_{j:\,h_j>0}\ |\rho_j^{\pm}|
    \;<\;
    \rho_0 ,
\label{eq:rate-hierarchy-discrete}
\end{equation}
i.e.\ every direction that changes a training prediction relaxes
strictly faster than every empirical-null direction.
\end{enumerate}
\end{theorem}

\begin{proof}
See Appendix~\ref{app:proofs-discrete}. The content of (iii) is a
monotonicity argument for the real branch and the exact modulus
$\sqrt\beta$ on the complex branch, together with
$\rho_0>\sqrt\beta$ from \cref{eq:overdamped-null}; the no-sign-flip
condition excludes a sliver adjacent to the stability boundary, where
strongly overdriven oscillatory modes can relax arbitrarily slowly.
\end{proof}

\begin{corollary}[Iteration-scale grokking clock]
\label{cor:iteration-clock}
Assume \cref{eq:overdamped-null} and write the null-block solution as
$\PN\xi_k=\rho_0^k\,w_0+\rho_1^k\,w_1$ with $w_0,w_1\in\Ncal_N$
determined by $(\xi_0,\xi_{-1})$. If $\Pg w_0\neq0$, then for
every $\varepsilon\in(0,\norm{\Pg w_0})$ the grokking iteration
count $k_{\mathrm{grok}}(\varepsilon):=
\min\{k:\norm{\Pg\xi_k}\le\varepsilon\}$ satisfies
\begin{equation}
    k_{\mathrm{grok}}(\varepsilon)
    =
    \frac{\log\bigl(\norm{\Pg w_0}/\varepsilon\bigr)}
         {-\log\rho_0}
    +O(1),
\label{eq:kgrok-exact}
\end{equation}
where the $O(1)$ term is controlled by the fast/slow amplitude ratio
and decays like $(\rho_1/\rho_0)^k$. In the weak-regularization regime
$\eta\lambda\ll(1-\beta)^2$ this gives the headline law
\begin{equation}
    \boxed{\ \
    k_{\mathrm{grok}}(\varepsilon)
    \;\approx\;
    \frac{1-\beta}{\eta\lambda}\,
    \log\frac{\norm{\Pg\xi_0}}{\varepsilon} .
    \ \ }
\label{eq:kgrok-headline}
\end{equation}
\end{corollary}

Equation~\eqref{eq:kgrok-headline} is stated in iterations, the
quantity actually measured in experiments, and requires no
continuous-time approximation: $\rho_0$ in \cref{eq:kgrok-exact} is
exact. A note on naming: $k_{\mathrm{grok}}$ is defined here as a
\emph{parameter-space relaxation count} of the population-active null
component, not through test accuracy; the name anticipates its
connection to the population-risk tail established in
Section~\ref{sec:population-tail}.

\begin{proposition}[Decoupled weight decay has a different clock]
\label{prop:decoupled}
For the decoupled variant of Remark~\ref{rem:coupled-decoupled}
applied to the linear model, the dynamics of a null coordinate
$q_k=\ip{\theta_k}{b}$, $b\in\Ncal_N$, is governed by the exact
two-dimensional linear system
\[
    q_{k+1}=(1-\eta\lambda)\,q_k+\beta v^b_k,
    \qquad
    v^b_{k+1}=\beta v^b_k ,
\]
with spectrum
\begin{equation}
    \{\,1-\eta\lambda,\ \beta\,\},
\label{eq:sgdw-spectrum}
\end{equation}
hence slow rate $-\log(1-\eta\lambda)\approx\eta\lambda$ and
\[
    k^{\mathrm{dec}}_{\mathrm{grok}}(\varepsilon)
    \approx
    \frac{1}{\eta\lambda}\,
    \log\frac{\norm{\Pg\xi_0}}{\varepsilon}
    \;=\;
    \frac{k_{\mathrm{grok}}(\varepsilon)}{1-\beta} .
\]
Coupled $L_2$ regularization therefore accelerates the grokking clock
by the factor $(1-\beta)^{-1}$ relative to decoupled weight decay at
equal $(\eta,\lambda,\beta)$: the momentum buffer integrates the decay
force in the coupled case but not in the decoupled one. This
optimizer dependence is a falsifiable prediction of the theory.
\end{proposition}

\begin{proof}
On $\Ncal_N$ we have $\nabla\loss_N(\theta)^\top b=\frac1N(\Phi\theta-y)^\top\Phi b=0$
exactly, so the loss never forces $v^b$; the update reads
$v^b_{k+1}=\beta v^b_k$ and
$q_{k+1}=q_k+v^b_{k+1}-\eta\lambda q_k$. The system matrix is
triangular with the stated eigenvalues.
\end{proof}

\subsection{Continuous-Time Limit: Exact Roots First}
\label{sec:continuous-roots}

The continuous-time companion of \cref{eq:null-recurrence} is the
damped oscillator obtained by restricting \cref{eq:damped-particle} to
a null coordinate $q=\ip{\theta}{b}$, $b\in\Ncal_N$:
\begin{equation}
    m\ddot q+\gamma\dot q+\lambda q=0 ,
\label{eq:null-oscillator}
\end{equation}
with \emph{exact} characteristic roots
\begin{equation}
    s_{\pm}
    =
    \frac{-\gamma\pm\sqrt{\gamma^2-4m\lambda}}{2m} .
\label{eq:continuous-roots}
\end{equation}
In the overdamped regime $\gamma^2\ge4m\lambda$ the slow decay rate is
\begin{equation}
    r_{\mathrm{slow}}
    :=
    -s_+
    =
    \frac{\gamma-\sqrt{\gamma^2-4m\lambda}}{2m}
    \;=\;
    \frac{2\lambda}{\gamma+\sqrt{\gamma^2-4m\lambda}} ,
\label{eq:rslow-exact}
\end{equation}
and only in the weak-regularization asymptotics $4m\lambda\ll\gamma^2$
does one recover the familiar rate
\begin{equation}
    r_{\mathrm{slow}}
    =
    \frac{\lambda}{\gamma}
    +\frac{m\lambda^2}{\gamma^3}
    +O\!\left(\frac{m^2\lambda^3}{\gamma^5}\right).
\label{eq:rslow-asymptotic}
\end{equation}
All continuous-time tails below are stated with the exact rate
$r_{\mathrm{slow}}$; expressions containing $\lambda/\gamma$ are
approximations of \cref{eq:rslow-exact}, not independent results.

\begin{remark}[Discrete--continuous consistency, leading order]
\label{rem:consistency}
With $m=(1+\beta)/2$, $\gamma=(1-\beta)/\sqrt\eta$, and the step
$\Delta t=\sqrt\eta$, the exact rate of the modified equation converts
to a per-step rate
$r_{\mathrm{slow}}\sqrt\eta=\eta\lambda/(1-\beta)+O\bigl((\eta\lambda)^2\bigr)$,
matching $-\log\rho_0$ from \cref{eq:slow-root-expansion} to
leading order; the two are distinct beyond that order, and the
numerical proximity of the two conversions in
Table~\ref{tab:rates} is a property of the chosen hyperparameters,
not an identity. 
\end{remark}

\section{Softly Broken Empirical Symmetries}
\label{sec:symmetries}

We now explain why the slow modes of
Section~\ref{sec:exact-linear} exist and what sets their rate: the
empirical null space carries an algebra of prediction-preserving
symmetries; weight decay supplies the only restoring force along
them (it controls the symmetry-breaking scale) while
inertia and damping convert that scale into a time or iteration rate.
The resulting charge dynamics is the slow-mode \cref{eq:null-oscillator}. Two distinct effects must be kept apart
throughout: friction damps every charge non-conservatively even at
$\lambda=0$, whereas weight decay is an explicit breaking of
the potential's symmetry; we use the term ``soft Noether law'' for the
resulting dissipative charge-balance equation.

\subsection{Hamiltonian Formulation of the Continuous Dynamics}
\label{sec:hamiltonian}

Throughout the main text the kinetic metric is the optimizer-induced
scalar $mI$, $m=(1+\beta)/2$; general constant metrics $M$, together
with the distinction between momentum-proportional and
velocity-proportional damping that arises for $M\neq mI$, are treated
in Appendix~\ref{app:noether}. On the phase space $T^\ast\Theta$ with
canonical momentum $p=m\dot\theta$, define
\begin{equation}
    \ham(\theta,p)
    =
    \frac{1}{2m}\,p^\top p+U(\theta).
\label{eq:hamiltonian}
\end{equation}
The continuous heavy-ball dynamics \cref{eq:damped-particle} is the
dissipative Hamiltonian flow
\begin{equation}
    \dot{\theta}
    =
    \frac{\partial \ham}{\partial p}
    =
    \frac{p}{m},
    \qquad
    \dot{p}
    =
    -\frac{\partial \ham}{\partial\theta}
    -
    \frac{\gamma}{m}\,p ,
\label{eq:hamiltonian-flow}
\end{equation}
along which
$\frac{d\ham}{dt}=-\gamma\norm{\dot\theta}^2
=-\frac{\gamma}{m^2}\norm{p}^2\le0$. 

For a vector field $X(\theta)$ on $\Theta$, the associated charge is
\[
    Q_X(\theta,p)
    =
    \ip{p}{X(\theta)} .
\]
Along \cref{eq:hamiltonian-flow} it satisfies (Appendix~\ref{app:noether})
\begin{equation}
    \frac{dQ_X}{dt}
    =
    \{Q_X,\ham\}-\frac{\gamma}{m}\,Q_X ,
    \qquad
    \{Q_X,\ham\}
    =
    \frac{m}{2}\,\dot\theta^\top\!\bigl(DX+DX^\top\bigr)\dot\theta
    -\ip{X}{\nabla U},
\label{eq:charge-drift}
\end{equation}
where $DX$ is the Jacobian of $X$. A vector field that generates a
one-parameter group of isometries of $mI$ (a Killing field,
$DX+DX^\top=0$) and preserves $\loss_N$ contributes only the friction
and weight-decay drifts.

\begin{lemma}[Noether condition, conservative case]
\label{lem:noether}
For the conservative, unregularized dynamics
($\gamma=0$, $\lambda=0$), $Q_X$ is conserved along every trajectory
if and only if $X$ is a Killing field of the kinetic metric and
$\ip{X}{\nabla\loss_N}\equiv0$, i.e.\ the flow of $X$ preserves both
the kinetic energy and the empirical loss.
\end{lemma}

\subsection{A Taxonomy: Null Directions, Local Symmetries, and Gauge Redundancy}
\label{sec:taxonomy}

A null direction of a Gram matrix at a single parameter value is sometimes informally interpreted as generating a symmetry. In general, however, this is only a local first-order condition and need not integrate to a finite prediction-preserving transformation. We will therefore distinguish these notions carefully.

\begin{definition}[Symmetry taxonomy]
\label{def:taxonomy}
Let $\mathcal U\subseteq\Theta$ be open, $\theta_\star\in\mathcal U$,
$v\in\R^P$, and let $X$ be a vector field on $\mathcal U$ with flow
$\Psi_s$.
\begin{enumerate}
\item[(a)] $v$ is an empirical-null tangent direction (a local
first-order symmetry at $\theta_\star$) if
$\mathsf J_N(\theta_\star)v=0$.
\item[(b)] $X$ is a local empirical symmetry field on
$\mathcal U$ if $\mathsf J_N(\theta)X(\theta)=0$ for all
$\theta\in\mathcal U$.
\item[(c)] $X$ is an integrable empirical symmetry if
$F_N(\Psi_s(\theta))=F_N(\theta)$ for all $\theta\in\mathcal U$ and
all $s$ for which the flow stays in $\mathcal U$.
\item[(d)] $v$ (or $X$) is a gauge redundancy if the
corresponding transformation preserves the network function itself,
$f(x;\Psi_s(\theta))=f(x;\theta)$ for $\mu$-a.e.\ $x$.
\end{enumerate}
\end{definition}

Property (b) implies (c) by integrating
$\frac{d}{ds}F_N(\Psi_s(\theta))=\mathsf J_N(\Psi_s(\theta))X(\Psi_s(\theta))=0$,
but (a) alone does not imply (c) for nonlinear models: a
Jacobian null direction at a single point need not integrate to a
finite transformation that preserves the predictions. Accordingly, for
nonlinear networks we use the terms \emph{empirical-null direction} and
\emph{population-null direction} for elements of $\Ker G_N(\theta_\star)$
and $\Ker G_\pop(\theta_\star)$, and reserve ``symmetry'' and ``gauge''
for cases where (b)--(d) actually hold; the local analysis of
Section~\ref{sec:nonlinear} quantifies the error committed by the
first-order picture.

For the exact linear model the gap closes: constant vector fields
$X\equiv b$ with $b\in\Ncal_N$ satisfy (b) globally, hence (c):
translations along the empirical null space are exact symmetries
of the prediction map itself, hence of the training loss,
$F_N(\theta+sb)=F_N(\theta)$ and
$\loss_N(\theta+sb)=\loss_N(\theta)$ for all $s\in\R$. Likewise $b\in\Ncal_\pop$ gives an exact gauge redundancy
(d), since $f(x;\theta+sb)=f(x;\theta)$ for $\mu$-a.e.\ $x$. Directions
in $\Scal=\Ncal_N\cap\Ncal_\pop^\perp$ are exact training symmetries
that are not gauge redundancies: they change the function on a set of positive measure. This is what makes them the carriers of grokking.

\subsection{The Affine Symmetry Algebra of the Linear Model}
\label{sec:affine-algebra}

To capture both \emph{translations} along the interpolation manifold and \emph{rotations} within the null space, we work with affine vector fields rather than restricting to homogeneous linear vector fields.

\begin{theorem}[Prediction-preserving affine symmetries, Euclidean metric]
\label{thm:affine-generators}
For the linear model with kinetic metric $mI$, consider the affine
vector field
\[
    X_{b,A}(\theta)=b+A(\theta-\theta_\eq).
\]
Then $X_{b,A}$ is a Killing field of $mI$ and preserves the empirical
prediction map pointwise, $F_N(\Psi_s(\theta))=F_N(\theta)$ for all
$\theta$ and $s$ (hence in particular preserves $\loss_N$) if
and only if
\begin{equation}
    \mathsf J_N b=0,
    \qquad
    \mathsf J_N A=0,
    \qquad
    A^\top+A=0 ,
\label{eq:affine-conditions}
\end{equation}
equivalently $b\in\Ncal_N$ and $A=V_0\,\Omega\,V_0^\top$ with
$\Omega^\top=-\Omega$, where $V_0\in\R^{P\times r_N}$ is any matrix
whose columns form an orthonormal basis of $\Ncal_N$. The set of such
fields is a Lie algebra under the vector-field bracket,
\[
    [X_{b,A},X_{b',A'}]
    =
    X_{A'b-Ab',\,A'A-AA'} ,
\]
isomorphic to the Euclidean algebra of the null space,
\begin{equation}
    \euc(r_N)\;=\;\R^{r_N}\rtimes\so(r_N),
    \qquad
    \dim=\frac{r_N(r_N+1)}{2} .
\label{eq:euclidean-algebra}
\end{equation}
\end{theorem}

\begin{proof}
The flow of $X_{b,A}$ preserves $F_N$ pointwise iff
$\frac{d}{ds}F_N(\Psi_s(\theta))=\mathsf J_NX_{b,A}(\Psi_s(\theta))=0$
for all $\theta,s$, i.e.\ iff $\mathsf J_NX_{b,A}\equiv0$, which
splits into $\mathsf J_Nb=0$ and $\mathsf J_NA=0$; for the constant
metric $mI$ the Killing condition is $DX+DX^\top=A+A^\top=0$
(translations are automatically Killing). Antisymmetry forces
$\Ran(A)=\Ran(A^\top)$, so $\mathsf J_NA=0$ (i.e.\
$\Ran(A)\subseteq\Ncal_N$) localizes both the range and the support of
$A$ to $\Ncal_N$, giving $A=V_0\Omega V_0^\top$. The bracket of two
affine fields $[X,Y]=DY\cdot X-DX\cdot Y$ is again affine with the
stated components, and the map
$(b,A)\mapsto X_{b,A}$ identifies the algebra with
$\R^{r_N}\rtimes\so(r_N)$. The general-metric version, in which
rotations satisfy the mass-weighted skew condition
$A^\top M+MA=0$ and are parametrized through pseudoinverse square
roots of $M$, is Theorem~\ref{thm:general-generators} in
Appendix~\ref{app:noether}.
\end{proof}

\begin{remark}[Loss symmetries are strictly larger]
\label{rem:loss-symmetries}
Theorem~\ref{thm:affine-generators} characterizes the symmetries of
the \emph{prediction map} $F_N$, not all symmetries of the scalar
loss $\loss_N$. The latter form a strictly larger algebra whenever
$G_N$ has a repeated positive eigenvalue: a rotation inside a
degenerate transverse eigenspace changes $F_N$ but preserves the
residual norm, hence $\loss_N$, around an appropriate center, and
such generators violate $\mathsf J_NA=0$. All conclusions below use
only the prediction-preserving subalgebra, which is what the Noether analysis of the null coordinates requires. We make no claim of classifying the full loss-symmetry algebra.
\end{remark}

The translation part $\R^{r_N}$ is the geometrically primary piece: it
moves along the interpolation manifold. The rotations $\so(r_N)$ act
inside the null space around $\theta_\eq$ and will supply the fast
angular clock.

\subsection{The Soft Noether Law Is the Slow-Mode Equation}
\label{sec:soft-noether}

\begin{theorem}[Soft breaking of the translation charges]
\label{thm:soft-noether}
Let $b\in\Ncal_N$ with $\norm b=1$, and define the translation charge
and the associated null coordinate
\[
    Q_b:=\ip{p}{b}=m\dot q_b,
    \qquad
    q_b(t):=\ip{\theta(t)}{b} .
\]
Along the flow \cref{eq:hamiltonian-flow} of the linear model,
\begin{equation}
    \frac{dQ_b}{dt}
    =
    -\frac{\gamma}{m}\,Q_b
    -
    \lambda\,q_b ,
\label{eq:soft-noether-translation}
\end{equation}
equivalently
\begin{equation}
    m\ddot q_b+\gamma\dot q_b+\lambda q_b=0 .
\label{eq:charge-oscillator}
\end{equation}
Thus the softly broken conservation law of the translation charge is
literally the null-mode oscillator \cref{eq:null-oscillator}: its exact
rates are $s_\pm$ of \cref{eq:continuous-roots}, and its discrete
counterpart is the exact recurrence \cref{eq:null-recurrence} with slow
root $\rho_0$.
\end{theorem}

\begin{proof}
By \cref{eq:charge-drift} with $DX=0$:
$\dot Q_b=-\ip{b}{\nabla U}-\frac{\gamma}{m}Q_b$, and
$\ip{b}{\nabla U(\theta)}
=\ip{b}{\nabla\loss_N(\theta)}+\lambda\ip{b}{\theta}
=0+\lambda q_b$, since
$\ip{b}{\nabla\loss_N(\theta)}=\frac1N(\Phi\theta-y)^\top\Phi b=0$
holds identically for $b\in\Ncal_N$.
\end{proof}

Three readings of \cref{eq:charge-oscillator} organize the paper:
\begin{enumerate}
\item[(i)] \emph{Unbroken limit.} For $\lambda=0$ the charge decays as
$Q_b(t)=Q_b(0)e^{-\gamma t/m}$, a purely dissipative effect of friction rather than symmetry breaking. The coordinate $q_b$ then freezes at a finite value: without weight decay, empirical-null
components never contract and the delay is infinite.
\item[(ii)] \emph{Soft breaking.} Weight decay adds the explicit
breaking term $-\lambda q_b$, the only restoring force along
$\Ncal_N$; the coordinate then contracts at the exact slow rate
$r_{\mathrm{slow}}$ of \cref{eq:rslow-exact}. In words: $\lambda$
sets the restoring \emph{scale}, while $m$ and $\gamma$ (equivalently
$\eta$ and $\beta$) determine its conversion into a rate.
\item[(iii)] \emph{Population split.} For $b\in\Ncal_\pop$ the
contraction is invisible in function space (a gauge motion); for
$b\in\Scal$ it changes the function on a set of positive
$\mu$-measure. Grokking is the soft breaking of the
\emph{population-active} translations.
\end{enumerate}

\begin{corollary}[Rotation charges: the angular clock]
\label{cor:rotation-charges}
Let $A=V_0\Omega V_0^\top$ as in
Theorem~\ref{thm:affine-generators} and
$Q_\Omega:=\ip{p}{A(\theta-\theta_\eq)}$. Then,
\begin{equation}
    \frac{dQ_\Omega}{dt}
    =
    -\frac{\gamma}{m}\,Q_\Omega ,
    \qquad\text{so}\qquad
    Q_\Omega(t)=Q_\Omega(t_0)\,e^{-\frac{\gamma}{m}(t-t_0)} :
\label{eq:angular-decay}
\end{equation}
angular motion relaxes on the clock $\tau_{\mathrm{ang}}=m/\gamma$,
independent of $\lambda$: the rotation charges are killed by
friction alone, not by the symmetry-breaking term.
\end{corollary}

\begin{proof}
By \cref{eq:charge-drift}, $\dot Q_\Omega=-\ip{A\xi}{\nabla U}
-\frac{\gamma}{m}Q_\Omega$ with $\xi=\theta-\theta_\eq$. Since
$\Ran(A)\subseteq\Ncal_N$, $\ip{A\xi}{\nabla\loss_N}=0$ identically,
and the weight-decay term is
$\lambda\ip{\theta}{A\xi}
=\lambda\ip{\theta_\eq}{A\xi}+\lambda\ip{\xi}{A\xi}=0$:
the first summand vanishes because $A\xi\in\Ncal_N$ and
$\theta_\eq\in\Ncal_N^\perp$ (\cref{eq:ridge-equilibrium}), the second
by antisymmetry of $A$. No centering assumption is needed; the
identity $\PN\theta_\eq=0$ does the work.
\end{proof}

\begin{remark}[Exact discrete angular clock]
\label{rem:discrete-angular}
The angular clock has an exact discrete counterpart. For two null
coordinates $q_{a,k},q_{b,k}$ obeying \cref{eq:null-recurrence}, the
discrete rotation charge given by the Casoratian
\[
    Q_k
    :=
    q_{a,k}\,q_{b,k-1}-q_{b,k}\,q_{a,k-1}
\]
satisfies $Q_{k+1}=\beta\,Q_k$, i.e.\
$Q_k=\beta^{\,k-1}Q_1$: the product of the characteristic roots is
$\rho_0\rho_1=\beta$ regardless of $\eta\lambda$. The per-step angular
rate $-\log\beta$ matches the continuous rate
$(\gamma/m)\Delta t=2(1-\beta)/(1+\beta)$ as $\beta\to1$. On the slow
manifold all null coordinates decay at the common rate $\rho_0$, so
rotation charges are carried entirely by the fast transient: angular
motion dies long before the grokking coordinate moves appreciably.
This is verified to high precision in Section~\ref{sec:experiments}.
\end{remark}

\begin{remark}[Isotropy of the angular multiplet]
\label{rem:multiplet}
The $X_{0,A}$ realize $\so(r_N)$,
$[X_\Omega,X_{\Omega'}]=X_{[\Omega,\Omega']}$; isotropic friction is
scalar on the multiplet, so all $\binom{r_N}{2}$ rotation charges decay
at the single rate $\gamma/m$. There is one angular clock, not a
spectrum.
\end{remark}

\section{Population-Risk Tail and the Grokking Time}
\label{sec:population-tail}

We now compute the late-time behavior of the population risk along the
exact trajectories of Section~\ref{sec:exact-linear}. Two framing
points first, both essential to reading the results correctly.
First, the quantity analyzed is the signed excess
$\loss_\pop(\theta_k)-\loss_\pop(\theta_\eq)$ relative to the
regularized training equilibrium: $\theta_\eq$ is not, in general, a
population optimum or even a population stationary point, so a
decaying excess means relaxation toward $\theta_\eq$, not
automatically an improvement over all comparators; the connection to
the population optimum requires the extra hypotheses of
Corollary~\ref{cor:monotone}. Second, ``only the grokking
subspace contributes'' is a statement about the \emph{slow
empirical-null component} of the asymptotic tail; transverse
directions still affect the population risk at finite times, through
the faster remainder. We collect the standing assumptions explicitly.

\begin{assumption}[Exact linear regime]
\label{ass:linear-regime}
The model is linear \cref{eq:linear-model} with squared loss
\cref{eq:losses}; $y_\star\in L^2_\mu$, so that $\loss_\pop$ is exactly
quadratic with $\nabla^2\loss_\pop=G_\pop$.
\end{assumption}

\begin{assumption}[Sampling]
\label{ass:sampling}
$\Ncal_\pop\subseteq\Ncal_N$; by Lemma~\ref{lem:kernel-inclusion} this
holds almost surely when the training inputs are drawn i.i.d.\ from
$\mu$, and we adopt it on that event.
\end{assumption}

\begin{assumption}[Weak regularization and no sign flips]
\label{ass:weak-reg}
$\eta\lambda<(1-\sqrt\beta)^2$ and
$\eta(h_j+\lambda)\le(1+\sqrt\beta)^2$ for all transverse eigenvalues
$h_j>0$ of $G_N$.
\end{assumption}

\begin{assumption}[Tail separation]
\label{ass:tail-separation}
$\varrho:=\max\bigl(\rho_1,\ \max_{j:h_j>0}|\rho_j^\pm|\bigr)<\rho_0^2$.
In the weak-regularization asymptotics this is implied by
$2\lambda\lesssim h_{\min}$ together with
$2\eta\lambda/(1-\beta)<-\log\beta$, i.e.\ weight decay is weak
relative to both the empirical curvature and the momentum friction.
\end{assumption}

Under Assumption~\ref{ass:weak-reg}, Theorem~\ref{thm:discrete-spectrum}
gives the strict hierarchy $|\rho_j^\pm|<\rho_0$;
Assumption~\ref{ass:tail-separation} strengthens it so that all fast
contributions decay below even the squared slow factor
$\rho_0^{2k}$, which is what a clean tail statement requires.

\begin{theorem}[Population-risk tail from the kernel mismatch]
\label{thm:population-tail}
Let Assumptions~\ref{ass:linear-regime}--\ref{ass:tail-separation}
hold, and write the null-block solution of
Theorem~\ref{thm:discrete-spectrum} as
$\PN\xi_k=\rho_0^kw_0+\rho_1^kw_1$. Define
\[
    \nabla_\pop:=\nabla\loss_\pop(\theta_\eq)
    =G_\pop\theta_\eq-g_\pop
    \ \in\ \Ncal_\pop^{\perp},
    \qquad
    s_0:=\Pg w_0 .
\]
Then there are constants $C>0$ (explicit in the proof) such that for
all $k\ge0$
\begin{equation}
    \loss_\pop(\theta_k)-\loss_\pop(\theta_\eq)
    =
    A\,\rho_0^{\,k}
    +
    B\,\rho_0^{\,2k}
    +
    R_k,
    \qquad
    |R_k|\le C\varrho^{\,k} ,
\label{eq:population-tail}
\end{equation}
with
\begin{equation}
    A=\ip{\nabla_\pop}{s_0}
    =\ip{\Pg\nabla_\pop}{w_0},
    \qquad
    B=\tfrac12\,s_0^\top G_\pop\,s_0\;\ge\;0 ,
\label{eq:tail-coefficients}
\end{equation}
and $B>0$ whenever $s_0\neq0$, since $G_\pop$ is positive definite on
$\Scal$. Moreover the training loss contains \emph{no} slow component
at all:
\begin{equation}
    \loss_N(\theta_k)-\loss_N(\theta_\eq)
    =
    O\!\bigl(\bar\rho^{\,k}\bigr),
    \qquad
    \bar\rho:=\max_{j:h_j>0}|\rho_j^\pm|<\rho_0 ,
\label{eq:training-blind}
\end{equation}
because $G_Nw=0$ and
$\ip{\nabla\loss_N(\theta_\eq)}{w}=-\lambda\ip{\theta_\eq}{w}=0$ for
every $w\in\Ncal_N$. Two regimes follow:
\begin{enumerate}
\item[(i)] \textbf{Generic tail.} If
$\Pg\nabla_\pop\neq0$ and $A\neq0$, the signed excess is
$A\rho_0^{k}(1+o(1))$; $A$ may take either sign, so neither the sign
nor the monotonicity of the approach is guaranteed: the
population risk can cross its limiting value and approach it from
below. ``Decay'' here refers to $|{\cdot}|\to0$, not to monotone
improvement.
\item[(ii)] \textbf{Compatible tail.} If the \emph{compatibility
condition}
\begin{equation}
    \Pg\nabla\loss_\pop(\theta_\eq)=0
\label{eq:compatibility}
\end{equation}
holds, then $A=0$ for every initial condition and
\begin{equation}
    \loss_\pop(\theta_k)-\loss_\pop(\theta_\eq)
    =
    B\,\rho_0^{\,2k}
    +O(\varrho^{\,k}),
    \qquad
    B=\tfrac12\,s_0^\top G_\pop s_0 .
\label{eq:compatible-tail}
\end{equation}
\end{enumerate}
\end{theorem}

\begin{proof}[Proof sketch; details in Appendix~\ref{app:proofs-tail}]
$\loss_\pop$ is exactly quadratic, so
$\loss_\pop(\theta_k)-\loss_\pop(\theta_\eq)
=\ip{\nabla_\pop}{\xi_k}+\frac12\xi_k^\top G_\pop\xi_k$ with no
remainder. Insert the exact mode decomposition. In the linear term,
$\nabla_\pop\perp\Ncal_\pop$ kills the population-null part of $w_0$,
leaving $A\rho_0^k$ plus terms bounded by $\varrho^k$ (the $w_1$
contribution and the transverse modes). In the quadratic term, the
slow--slow contribution is
$\frac12w_0^\top G_\pop w_0\,\rho_0^{2k}=B\rho_0^{2k}$ because
$G_\pop$ annihilates the $\Ncal_\pop$-component of $w_0$ and
$\Scal\perp\Ncal_\pop$ removes the cross term. Every remaining product
of mode factors is one of
$\rho_1^k,\ \bar\rho^{\,k},\ (\rho_0\rho_1)^k,\ (\rho_0\bar\rho)^k,\
\rho_1^{2k},\ (\rho_1\bar\rho)^k,\ \bar\rho^{\,2k}$, and each is
bounded by $\varrho^{\,k}$: indeed $\rho_1\le\varrho$ and
$\bar\rho\le\varrho$ by definition, $\rho_0\rho_1=\beta\le\rho_1$, and
$\rho_0\bar\rho<\bar\rho$ since $\rho_0<1$. The role of
Assumption~\ref{ass:tail-separation} is only to make
$\varrho^{\,k}=o(\rho_0^{2k})$, so that $R_k$ is subleading.
Positive definiteness of $G_\pop$ on $\Scal$: for
$s\in\Scal\setminus\{0\}$, $s\perp\Ncal_\pop=\Ker G_\pop$ gives
$s^\top G_\pop s>0$.
\end{proof}

\begin{remark}[When does compatibility hold?]
\label{rem:compatibility}
Condition \cref{eq:compatibility} formalizes the statement ``the
regularized training equilibrium is population-stationary along the
grokking subspace.'' It is an assumption about the data distribution,
not a consequence of optimization, and we keep it visible rather than
folding it into the setup. Two useful sufficient conditions:
\begin{enumerate}
\item[(a)] \emph{Realizable, training-identified teacher.} If
$y_i=\phi_i^\top\theta_T$ and $y_\star=\phi^\top\theta_T$ with
$\theta_T\in\Ncal_N^\perp$, then
$\nabla_\pop=G_\pop(\theta_\eq-\theta_T)$ with
$\theta_\eq-\theta_T=-\lambda(G_N+\lambda I)^{-1}\theta_T=O(\lambda)$,
so $A=O(\lambda\norm{\theta_T}\norm{w_0})$: the linear term is
weak-regularization suppressed, and the $\rho_0^{2k}$ law of
\cref{eq:compatible-tail} holds over $\sim\log(1/\lambda)$ decades
before the residual $\rho_0^k$ term can surface.
\item[(b)] \emph{Isotropic population geometry.} If
$G_\pop$ acts as a multiple of the identity on
$\Ran G_\pop\supseteq\Scal\oplus(\Ncal_N^\perp\cap\Ran G_\pop)$ (e.g.\ whitened features), then
$\nabla_\pop\in\Ncal_N^\perp$ implies
$\Pg\nabla_\pop=0$. This is the regime of the synthetic
experiments in Section~\ref{sec:experiments}, where
\cref{eq:compatibility} holds with no $O(\lambda)$ residue.
\end{enumerate}
Without any such condition, Theorem~\ref{thm:population-tail}(i) still
applies, but only in the weaker sense that the \emph{signed} excess is
dominated by a $\rho_0^{k}$ term whose sign and monotonicity depend on
$A$. The dichotomy itself ($\rho_0^{2k}$ versus $\rho_0^{k}$,
decided by $\Pg\nabla\loss_\pop(\theta_\eq)$) is a falsifiable
prediction, verified in Figure~\ref{fig:rates}(b).
\end{remark}

\begin{corollary}[Excess-risk relaxation time on the weight-decay clock]
\label{cor:grokking-time}
Assume the setting of Theorem~\ref{thm:population-tail}(ii) and
$s_0\neq0$. Then the number of iterations after which the excess
population risk (relative to $\theta_\eq$; this is an excess-risk
threshold, not an accuracy threshold) stays below $\varepsilon$ obeys
\begin{equation}
    k_{\mathrm{grok}}^{\pop}(\varepsilon)
    =
    \frac{\log\bigl(B/\varepsilon\bigr)}{2\,(-\log\rho_0)}
    +O(1)
    \;\approx\;
    \frac{1-\beta}{2\eta\lambda}\,
    \log\frac{B}{\varepsilon} ,
\label{eq:kgrok-pop}
\end{equation}
consistent with the parameter-space clock
\cref{eq:kgrok-headline} (the factor $2$ converts a threshold on the
squared quantity into a threshold on the norm). In continuous time,
\begin{equation}
    \tau_{\mathrm{grok}}(\varepsilon)
    =
    \frac{1}{2\,r_{\mathrm{slow}}}\,
    \log\frac{B}{\varepsilon}
    \;=\;
    \frac{\gamma}{2\lambda}\,
    \log\frac{B}{\varepsilon}\,
    \bigl(1+O(m\lambda/\gamma^2)\bigr)
    \;\propto\;
    \frac{1-\beta}{\lambda\sqrt\eta}\,,
\label{eq:tau-grok}
\end{equation}
so the relaxation is slow when weight decay is weak, and the
delay diverges as $\lambda\to0$, when the softly broken translation
symmetry is restored.
\end{corollary}

Corollary~\ref{cor:grokking-time} concerns relaxation toward
$\theta_\eq$. The bridge to actual generalization improvement
requires knowing where $\theta_\eq$ stands relative to the population
optimum; the following realizable case is the cleanest.

\begin{corollary}[Monotone approach to the population optimum,
realizable case]
\label{cor:monotone}
In the setting of Remark~\ref{rem:compatibility}(a) (realizable,
training-identified teacher: $y_i=\phi_i^\top\theta_T$,
$y_\star=\phi^\top\theta_T$, $\theta_T\in\Ncal_N^\perp$), the
population optimum is $\loss_\pop^\star=0$ and
\begin{equation}
    \loss_\pop(\theta_k)-\loss_\pop^\star
    =
    \underbrace{\loss_\pop(\theta_\eq)}_{O(\lambda^2\norm{\theta_T}^2)}
    \;+\;
    B\,\rho_0^{2k}
    \;+\;
    O(\lambda)\,\rho_0^{k}
    \;+\;
    O(\varrho^{\,k}),
\label{eq:optimum-excess}
\end{equation}
with
$\loss_\pop(\theta_\eq)
\le\tfrac12\lambda_{\max}(G_\pop)\,
\lambda^2\norm{(G_N+\lambda I)^{-1}\theta_T}^2$. Three consequences:
the excess over the true optimum is nonnegative throughout (since
$\loss_\pop^\star$ is the global minimum); it decays to an
$O(\lambda^2)$ ridge floor; and the decay follows the doubled-rate
law $B\rho_0^{2k}$ over the $\sim\log(1/\lambda)$ decades during
which $B\rho_0^{2k}\gtrsim\lambda\rho_0^{k}$, after which the
weak-$\lambda$-suppressed $\rho_0^{k}$ correction governs the final
approach to the floor. Outside the realizable case no improvement
claim relative to the optimum is made.
\end{corollary}

\begin{proof}
$\loss_\pop(\theta)=\tfrac12(\theta-\theta_T)^\top G_\pop(\theta-\theta_T)$
and $\theta_\eq-\theta_T=-\lambda(G_N+\lambda I)^{-1}\theta_T$ give
the floor bound; adding \cref{eq:population-tail} with
$A=O(\lambda\norm{\theta_T}\norm{w_0})$
(Remark~\ref{rem:compatibility}(a)) gives
\cref{eq:optimum-excess}, and nonnegativity is immediate from
$\loss_\pop\ge0=\loss_\pop^\star$. The crossover between the
$\rho_0^{2k}$ and $\lambda\rho_0^{k}$ terms occurs at
$\rho_0^{k}\asymp\lambda/B$, i.e.\ after
$\log(B/\lambda)/(-\log\rho_0)$ iterations.
\end{proof}

The three clocks of the theory are summarized, under
Assumptions~\ref{ass:weak-reg}--\ref{ass:tail-separation} and in the
overdamped transverse regime $\gamma^2\ge4m(h_j+\lambda)$, by
\begin{equation}
    \tau_{\mathrm{ang}}=\frac{m}{\gamma}
    \;\le\;
    \tau_{\perp,j}\simeq\frac{\gamma}{h_j+\lambda}
    \;\ll\;
    \tau_{\parallel}=\frac{1}{r_{\mathrm{slow}}}
    \simeq\frac{\gamma}{\lambda} ,
\label{eq:three-clocks}
\end{equation}
where the middle clock is replaced by $2m/\gamma$ in the underdamped
case, which only sharpens the separation, and the last separation
requires $\lambda\ll h_j$. The training loss runs on the first two
clocks and is blind to the third \cref{eq:training-blind};
the population loss runs on the third through its grokking component
alone \cref{eq:population-tail}.

\begin{remark}[Causal and falsifiable predictions]
\label{rem:predictions}
The theory makes four predictions that go beyond the
$\tau\propto\gamma/\lambda$ scaling shared with earlier accounts:
\begin{enumerate}
\item[(P1)] \emph{Collapse.} Measured grokking iteration counts
across $(\eta,\lambda,\beta)$ collapse onto the exact prediction
$\log(\norm{q_0}/\varepsilon)/(-\log\rho_0)$ with $\rho_0$ from
\cref{eq:slow-root}. A power law in $\lambda$ alone does not reproduce the collapse.
\item[(P2)] \emph{Intervention.} Rescaling the grokking component
$\Pg\xi$ at any post-interpolation time by $c$ leaves every
training prediction unchanged and shifts the grokking iteration by
exactly $\log c/(-\log\rho_0)$.
\item[(P3)] \emph{Gauge control.} Perturbations along $\Ncal_\pop$
change neither training nor population predictions; perturbations
along $\Scal$ change population predictions only; perturbations along
$\Ncal_N^\perp$ change both.
\item[(P4)] \emph{Optimizer dependence.} Replacing coupled $L_2$ by
decoupled weight decay multiplies the grokking iteration count by
$(1-\beta)^{-1}$ at fixed $(\eta,\lambda,\beta)$
(Proposition~\ref{prop:decoupled}).
\end{enumerate}
All four identities are verified, with no fitted parameters, in the
exactly solvable synthetic model of Section~\ref{sec:experiments};
the modular-addition benchmark of Section~\ref{sec:modular} provides
a first nonlinear test of (P1). Nonlinear validity of (P2)--(P4)
remains open.
\end{remark}

\section{Local Nonlinear Extension}
\label{sec:nonlinear}

For nonlinear networks the results above become local statements around
an interpolating region, and every first-order identity acquires a
remainder. This section states the assumptions under which the exact
linear theory survives as a controlled approximation, with the errors
displayed rather than absorbed silently.

\subsection{Assumptions with Explicit Remainders}

Let $\theta_\star$ be a reference parameter (the relevant choices were
separated in Section~\ref{sec:reference-points}) and
$\xi:=\theta-\theta_\star$. Stack the training outputs into
$F_N(\theta)\in\R^N$.

\begin{assumption}[Training-side local linearization, constant rank]
\label{ass:local-linear}
There are a neighborhood $\mathcal U$ of $\theta_\star$ and a constant
$C_N$ such that
\[
    \norm{F_N(\theta)-F_N(\theta_\star)-\mathsf J_N\,\xi}
    \le
    C_N\norm{\xi}^2,
    \qquad
    \mathsf J_N:=\nabla_\theta F_N(\theta_\star),
\]
for all $\theta\in\mathcal U$, and $\nabla_\theta F_N(\theta)$ has
constant rank on $\mathcal U$.
\end{assumption}

\begin{assumption}[Population-side local linearity]
\label{ass:pop-linear}
There is a measurable $c:\R^{n_0}\to[0,\infty)$ with
$\E_\mu[c(x)^2]=:C_\pop^2<\infty$ such that for $\mu$-a.e.\ $x$ and all
$\theta\in\mathcal U$,
\[
    \bigl|f(x;\theta)-f(x;\theta_\star)-J_x(\theta_\star)\,\xi\bigr|
    \le
    c(x)\norm{\xi}^2 .
\]
\end{assumption}

\begin{assumption}[Population residual]
\label{ass:pop-residual}
The population residual $r_\star(x):=f(x;\theta_\star)-y_\star(x)$
satisfies $\E_\mu[r_\star^2]=:\varepsilon_\star^2<\infty$.
\end{assumption}

Assumption~\ref{ass:pop-linear} is the population counterpart of
Assumption~\ref{ass:local-linear}: the
expansion of $\loss_\pop$ requires control of the linearization error
for $\mu$-almost every input, not only on the training set. Under
Assumptions~\ref{ass:pop-linear}--\ref{ass:pop-residual}, a direct
computation with Cauchy--Schwarz gives, for $\theta\in\mathcal U$,
\begin{equation}
    \Bigl|
    \loss_\pop(\theta)-\loss_\pop(\theta_\star)
    -\ip{\nabla\loss_\pop(\theta_\star)}{\xi}
    -\tfrac12\,\xi^\top G_\pop(\theta_\star)\,\xi
    \Bigr|
    \le
    \varepsilon_\star C_\pop\norm{\xi}^2
    +C'\norm{\xi}^3
    +\tfrac12C_\pop^2\norm{\xi}^4 ,
\label{eq:pop-expansion}
\end{equation}
with $C'=C_\pop\sqrt{\lambda_{\max}(G_\pop)}$: the population Hessian
agrees with the Gram matrix only up to the residual--curvature term,
$\nabla^2\loss_\pop(\theta_\star)=G_\pop+O(\varepsilon_\star C_\pop)$,
which vanishes in the population-realizable case
$\varepsilon_\star=0$. For the squared loss the same mechanism on the
training side gives $\nabla^2\loss_N(\theta_\star)=G_N$  at an
interpolator, since the residual factor is zero there.

\subsection{Perturbed Rates and Scope}

\begin{assumption}[Spectral gap and Jacobian regularity]
\label{ass:gap}
On $\mathcal U$, the nonzero singular values of
$\nabla_\theta F_N(\theta)$ are bounded below by $\sigma_0>0$, and
$\theta\mapsto\nabla_\theta F_N(\theta)$ is $L_J$-Lipschitz. (The
constant-rank hypothesis of Assumption~\ref{ass:local-linear} alone
does \emph{not} control the rotation of $\Ncal_N(\theta)$; the gap
$\sigma_0$ and the constant $L_J$ enter the Davis--Kahan-type bound
$\norm{\PN(\theta)-\PN(\theta_\star)}\le 2L_J\norm\xi/\sigma_0$ that
does.)
\end{assumption}

\begin{proposition}[Local relaxation of the grokking component, with
error floor]
\label{prop:rate-bracket}
Let Assumptions~\ref{ass:local-linear}--\ref{ass:gap} hold at
$\theta_\star=\theta_\eq$ (the minimizer of $U$, assumed to lie in
$\mathcal U$ with $\norm{\PN\theta_\eq}\le c_0\lambda$), and consider
the overdamped continuous dynamics on a window $[t_0,T]$ during which
$\sup_{t}\norm{\xi(t)}\le\delta$ and the transverse transient has
passed. Write $\zeta:=\Pg\xi$ with the projector frozen at
$\theta_\star$. Then there are constants
$C_\delta=O\bigl((C_N+C_\pop+L_J/\sigma_0)\,\delta+\lambda\delta\bigr)$
and
\[
    F=O\bigl(C_N\delta^2+c_0\lambda^2\bigr)
    \qquad(\text{the perturbation floor})
\]
such that, for $t\in[t_0,T]$,
\begin{equation}
    \norm{\zeta(t)}
    \;\le\;
    e^{-(r_{\mathrm{slow}}-C_\delta)(t-t_0)}\,\norm{\zeta(t_0)}
    \;+\;
    \frac{F}{\gamma\,(r_{\mathrm{slow}}-C_\delta)} ,
\label{eq:rate-upper}
\end{equation}
\begin{equation}
    \norm{\zeta(t)}
    \;\ge\;
    e^{-(r_{\mathrm{slow}}+C_\delta)(t-t_0)}\,\norm{\zeta(t_0)}
    \;-\;
    \frac{F}{\gamma\,(r_{\mathrm{slow}}+C_\delta)} .
\label{eq:rate-lower}
\end{equation}
Consequently the \emph{multiplicative} two-sided bracket
\begin{equation}
    e^{-(r_{\mathrm{slow}}+C_\delta)(t-t_0)}
    \lesssim
    \frac{\norm{\zeta(t)}}{\norm{\zeta(t_0)}}
    \lesssim
    e^{-(r_{\mathrm{slow}}-C_\delta)(t-t_0)}
\label{eq:rate-bracket}
\end{equation}
holds only \emph{while $\norm{\zeta(t)}$ exceeds a fixed multiple of
the perturbation floor} $F/(\gamma r_{\mathrm{slow}})$, and no
statement is made below the floor. The population expansion
\cref{eq:pop-expansion} then transfers these bounds to the
population-risk tail with multiplicative error
$1+O(\delta+\varepsilon_\star)$ on the same window. All of this is a
continuous-time statement: we claim no discrete-time nonlinear
analogue, and the discrete optimizer-dependence results of
Section~\ref{sec:exact-linear} are exact only for the linear model.
\end{proposition}

\begin{proof}[Proof sketch]
Project the dynamics onto $\Scal$ (projector frozen at
$\theta_\star$) and treat the nonlinear terms as a perturbation: on
$\Ncal_N$ the empirical force is
$\PN\nabla\loss_N(\theta)=O(C_N\norm\xi^2)$ by
Assumption~\ref{ass:local-linear} (it vanishes to first order), the
weight-decay force is
$-\lambda\PN\theta=-\lambda\PN\xi+O(c_0\lambda^2)$, and the
projector-drift coupling is bounded through
Assumption~\ref{ass:gap}. In the overdamped regime the slaved
first-order dynamics is $\gamma\dot\zeta=-\lambda\zeta+e(t)$ with
$\norm{e(t)}\le F+\gamma C_\delta\norm{\zeta}$. The differential
inequalities
$\bigl|\frac{d}{dt}\norm\zeta+r_{\mathrm{slow}}\norm\zeta\bigr|
\le C_\delta\norm\zeta+F/\gamma$ integrate, by Gr\"onwall and
variation of constants, to
\cref{eq:rate-upper}--\cref{eq:rate-lower}; the additive terms are
the integrated forcing and cannot be removed, which is why
\cref{eq:rate-bracket} is conditional on staying above the floor.
Appendix~\ref{app:proofs-tail} records the details.
\end{proof}

\begin{remark}[Scope]
\label{rem:scope}
The results of
Sections~\ref{sec:exact-linear}--\ref{sec:population-tail} are exact
for models linear in their parameters. For nonlinear networks they
are local, and
\cref{eq:rate-upper}--\cref{eq:rate-bracket} deliberately degrade
from an equality to conditional bounds with an explicit perturbation
floor: Jacobian variation, rotation of $\Ncal_N$, coupling among
angular, transverse, and null coordinates, and higher-order loss
terms enter through $C_\delta$ and $F$, and nothing is claimed below
the floor or in discrete time for nonlinear models. We do not claim
that this local picture explains grokking in full generality; we claim
that it isolates a mechanism (soft breaking of population-active
empirical-null translations) whose signatures are exactly
computable in the linear regime and testable beyond it
(Remark~\ref{rem:predictions}). The conversion from population-risk
decay to the abrupt jump in test \emph{accuracy} observed empirically
involves a margin argument that we leave open
(Section~\ref{sec:conclusion}).
\end{remark}

\section{Numerical Verification and a Nonlinear Benchmark}
\label{sec:experiments}

Our numerical work has two roles. E1--E3 below
are \emph{no-fitted-parameter verifications in the solvable model}:
the theory of
Sections~\ref{sec:exact-linear}--\ref{sec:population-tail} is exact,
every subspace, projector, root, and tail coefficient is computable in
closed form, and the runs check that theorems, algebra, and
implementation agree. They do not test whether the mechanism operates in deep networks; that is the role of Section~\ref{sec:modular}, which probes external validity on an actual grokking task.
Deviations quoted for E1--E3 are numerical-tolerance statements about
deterministic recurrences, not statistical claims. 
\subsection{An Exactly Computable Synthetic Model}
\label{sec:synthetic-setup}

We take $P=60$, features $\phi(x)=Wx$ with
$W\in\R^{P\times d}$ a fixed random matrix with orthonormal columns,
$d=20$, inputs $x\sim\Ncal(0,I_d)$, and $N=8$ i.i.d.\ training
inputs. Then $G_\pop=WW^\top$ is the orthogonal projection onto
$\Ran(W)$ and
\[
    \dim\Ncal_N^\perp=8,
    \qquad
    \dim\Scal=d-N=12,
    \qquad
    \dim\Ncal_\pop=P-d=40 ,
\]
with all three subspaces available in closed form
($\Scal=W\cdot\Ker X$ for the input matrix $X\in\R^{N\times d}$).
Because $G_\pop$ is a projection, it acts isotropically on its range,
so compatibility \cref{eq:compatibility} holds for any teacher
with $\theta_T\in\Ncal_N^\perp$ (Remark~\ref{rem:compatibility}(b));
a generic teacher $\theta_T=Wu_T\in\Ran W$ breaks it. Targets are
noiseless, $y_i=\phi_i^\top\theta_T$,
$y_\star=\phi^\top\theta_T$. Unless stated otherwise
$\eta=0.05$, $\lambda=2\times10^{-3}$, $\beta=0.95$, $\theta_0\sim\Ncal(0,I)$
with a warm momentum buffer $v_0\sim\Ncal(0,10^{-2}I)$ (a cold start
$v_0=0$ makes every rotation charge zero for all time,
since all null coordinates then move radially; this is itself a check of
Remark~\ref{rem:discrete-angular}), run for $10^4$ steps; then
$\eta\lambda=10^{-4}<(1-\sqrt\beta)^2\approx6.41\times10^{-4}$
(Assumption~\ref{ass:weak-reg}) and
Assumption~\ref{ass:tail-separation} holds with large margin.

\subsection{E1: Rate Hierarchy and the Two Population Tails}
\label{sec:e1}

Figure~\ref{fig:rates}(a) shows the exact rate hierarchy. The
transverse component collapses on the fast clock with the predicted
envelope $\beta^{\,k/2}$ (all transverse modes are underdamped at
these hyperparameters);
the grokking component $\norm{\Pg\xi_k}$ and the population-null
component $\norm{\Pzero\xi_k}$ decay side by side at the exact slow root
$\rho_0^{\,k}$. The two decay at the same rate in parameter space and differ only in their visibility in function space. The discrete rotation charge
follows the exact Casoratian law $\beta^{\,k}$ of
Remark~\ref{rem:discrete-angular} down to the numerical floor.
Table~\ref{tab:rates} compares the measured rates with the exact and
asymptotic predictions.

Figure~\ref{fig:rates}(b) is the grokking curve itself. The training
loss reaches its floor within a few hundred iterations and contains no
slow component, as \cref{eq:training-blind} requires. The excess
population risk of the compatible teacher decays on the weight-decay
clock at the doubled rate $\rho_0^{2k}$
\cref{eq:compatible-tail}; the generic teacher decays at
$\rho_0^{k}$ (Theorem~\ref{thm:population-tail}(i)). Both predicted
slopes are drawn without any fitting.

\begin{figure}[t]
\centering
\includegraphics[width=\textwidth]{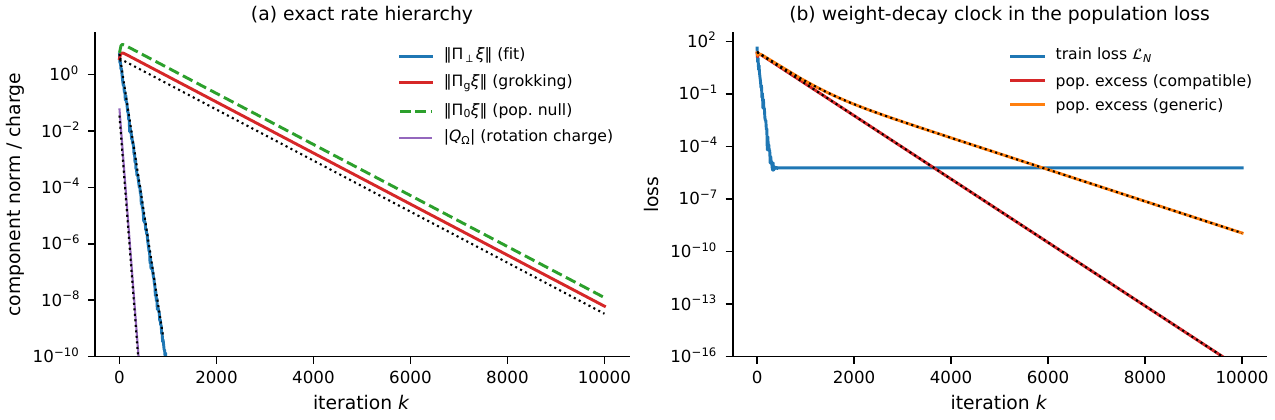}
\caption{E1 on the synthetic model
($P=60$, $d=20$, $N=8$, $\eta=0.05$, $\lambda=2\times10^{-3}$, $\beta=0.95$).
(a) Component norms and rotation charge against iteration; dotted
lines are the parameter-free predictions $\beta^{\,k/2}$,
$\rho_0^{\,k}$, and $\beta^{\,k}$.
(b) Training loss and excess population risk for the compatible and
generic teachers; dotted curves are the fully parameter-free
predictions $B\rho_0^{2k}$ and $|A\rho_0^{k}+B\rho_0^{2k}|$ with
$A,B$ computed from \cref{eq:tail-coefficients}; no quantity is fitted.}
\label{fig:rates}
\end{figure}

\begin{table}[t]
\centering
\caption{Rates per iteration in E1
($\eta=0.05$, $\lambda=2\times10^{-3}$, $\beta=0.95$): exact discrete
predictions, weak-regularization and continuous-time conversions, and
measured decay rates (late-window linear fits in log space).}
\label{tab:rates}
\small
\begin{tabular}{lcc}
\toprule
quantity & prediction & measured \\
\midrule
slow rate, exact $-\log\rho_0$ & $2.0847490\times10^{-3}$ & --- \\
slow rate, asymptotic $\eta\lambda/(1-\beta)$ & $2.0000000\times10^{-3}$ & --- \\
slow rate, continuous $r_{\mathrm{slow}}\sqrt\eta$ & $2.0847506\times10^{-3}$ & --- \\
$\norm{\Pg\xi}$ decay & $2.0847490\times10^{-3}$ & $2.0847494\times10^{-3}$ \\
$\norm{\Pzero\xi}$ decay & $2.0847490\times10^{-3}$ & $2.0847490\times10^{-3}$ \\
transverse envelope & $-\tfrac12\log\beta=2.5646647\times10^{-2}$ & $2.5569983\times10^{-2}$ \\
rotation charge $|Q_\Omega|$ decay & $-\log\beta=5.1293083\times10^{-2}$ & $5.1293294\times10^{-2}$ \\
population tail, compatible & $-2\log\rho_0=4.1694980\times10^{-3}$ & $4.1703601\times10^{-3}$ \\
population tail, generic & $-\log\rho_0=2.0847490\times10^{-3}$ & $2.0848660\times10^{-3}$ \\
\bottomrule
\end{tabular}
\end{table}

\subsection{E2: Scaling Sweep over $(\lambda,\eta,\beta)$}
\label{sec:e2}

We sweep $\lambda\in\{3\times10^{-4},10^{-3},3\times10^{-3}\}$,
$\eta\in\{0.02,0.05,0.1\}$,
$\beta\in\{0.5,0.8,0.9,0.95\}$ (36 configurations), measure the slow
decay rate of $\norm{\Pg\xi_k}$ by a late-window linear fit in
log space, and record the grokking iteration
$k_{\mathrm{grok}}$ at threshold
$\varepsilon=10^{-4}\norm{\Pg\xi_{k_0}}$ with $k_0=10^3$.
Figure~\ref{fig:scaling}(a) plots measured rates against the exact
$-\log\rho_0$: all 36 points lie on the identity line with maximal
relative deviation $8.4\times10^{-5}$, while the asymptotic rate
$\eta\lambda/(1-\beta)$ visibly deviates at the largest
$\eta\lambda/(1-\beta)^2$, by the amount the exact root predicts. Figure~\ref{fig:scaling}(b) shows the iteration-count
collapse of \cref{eq:kgrok-exact}: measured $k_{\mathrm{grok}}$
against $\log(\norm{q_0}/\varepsilon)/(-\log\rho_0)$ across all
configurations, with no free parameter; all points agree with the
prediction within $2.2\%$, the residual being the $O(1)$ transient
term of \cref{eq:kgrok-exact}. (These deviations quantify fit-window
and implementation tolerance in a deterministic recurrence; they are
identity checks, not statistical evidence.)

\begin{figure}[t]
\centering
\includegraphics[width=\textwidth]{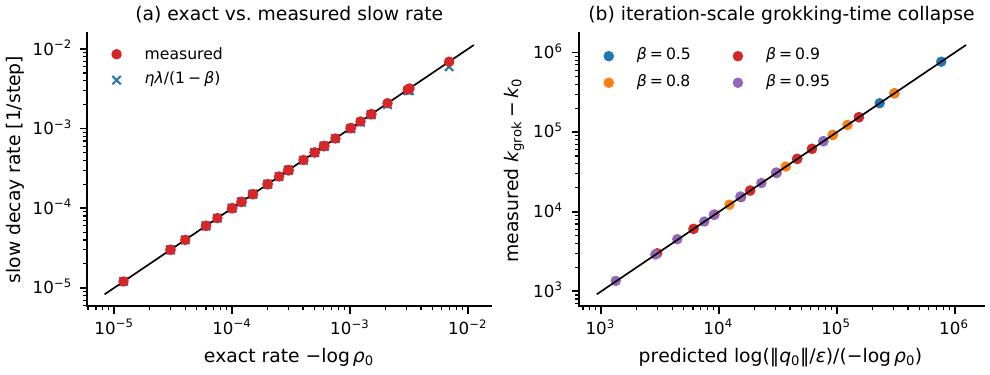}
\caption{E2 scaling sweep over 36 configurations of
$(\lambda,\eta,\beta)$.
(a) Measured slow rate versus the exact root $-\log\rho_0$ (identity
line); crosses show the asymptotic rate $\eta\lambda/(1-\beta)$,
which deviates where $\eta\lambda/(1-\beta)^2$ is not small.
(b) Grokking-iteration collapse against the exact prediction; the
line is the identity, not a fit.}
\label{fig:scaling}
\end{figure}

\subsection{E3: Interventions and Gauge Control}
\label{sec:e3}

\paragraph{Logarithmic delay shift (P2).}
At $k_0=3000$ (well after interpolation) we rescale the grokking
component of the state by $c\in\{0.03,\dots,30\}$, applying the rescaling to the $\Pg$-component of $\theta$ and of the momentum buffer so that the state stays on the slow branch; rescaling $\theta$ alone would excite
the fast null root $\rho_1$. The training predictions are unchanged to numerical precision (maximal $3.0\times10^{-14}$, double-precision roundoff). We then resume training. Figure~\ref{fig:intervention}(a) shows
$k_{\mathrm{grok}}$ against $c$: the shift is linear in $\log c$ with
measured slope $479.67405$ against the predicted
$1/(-\log\rho_0)=479.67405$.

\paragraph{Gauge control (P3).}
Unit perturbations at $k_0$ along $\Ncal_\pop$, $\Scal$, and
$\Ncal_N^\perp$ produce the predicted train/population signature
(Figure~\ref{fig:intervention}(b)): the $\Ncal_\pop$ direction changes
neither training predictions nor the function in $L^2_\mu$
(both at roundoff), the $\Scal$ direction changes only the population
predictions, and the transverse direction changes both. This is the
operational meaning of the decomposition
\cref{eq:decomposition}: ``flat direction'' is not one thing.

\begin{figure}[t]
\centering
\includegraphics[width=\textwidth]{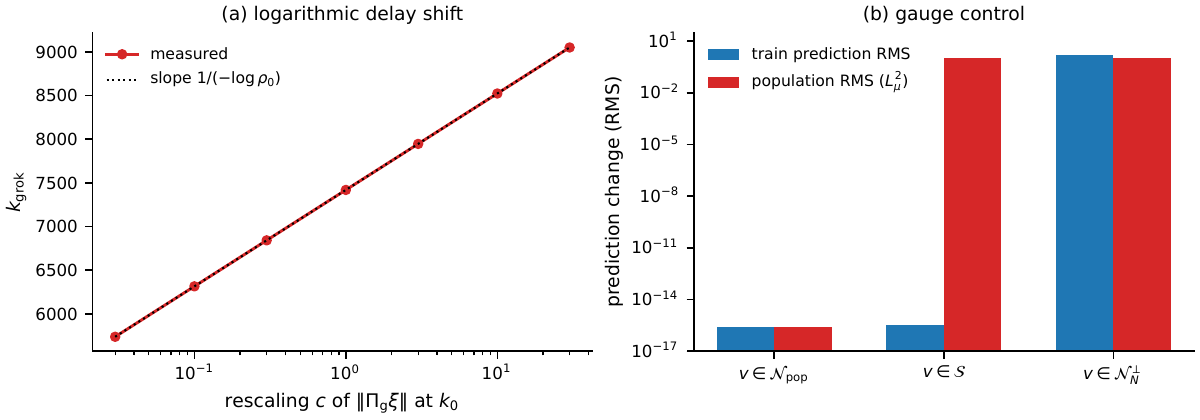}
\caption{E3 interventions.
(a) Rescaling the grokking component at $k_0$ shifts the grokking
iteration by $\log c/(-\log\rho_0)$ (dotted: predicted slope) while
leaving every training prediction unchanged.
(b) RMS change of training predictions and of the function in
$L^2_\mu$ under unit perturbations along $\Ncal_\pop$, $\Scal$, and
$\Ncal_N^\perp$; bars at $10^{-16}$ are exact zeros at roundoff.}
\label{fig:intervention}
\end{figure}

\paragraph{Optimizer dependence (P4).}
Replacing coupled $L_2$ by decoupled weight decay at
$(\eta,\lambda,\beta)=(0.05,2\times10^{-3},0.95)$ changes the measured slow
rate from $-\log\rho_0$ to $-\log(1-\eta\lambda)$, i.e.\ by the factor
$(1-\beta)$ predicted by Proposition~\ref{prop:decoupled}
(measured values in Table~\ref{tab:sgdw}).

\begin{table}[t]
\centering
\caption{P4: slow rate per iteration under coupled $L_2$
regularization versus decoupled weight decay at
$\eta=0.05$, $\lambda=2\times10^{-3}$, $\beta=0.95$; the ratio is
$(1-\beta)$ up to the $O(\eta\lambda)$ corrections of the exact
roots.}
\label{tab:sgdw}
\small
\begin{tabular}{lcc}
\toprule
optimizer & predicted slow rate & measured \\
\midrule
heavy ball + coupled $L_2$ & $-\log\rho_0=2.0847490\times10^{-3}$ & $2.0847490\times10^{-3}$ \\
heavy ball + decoupled decay & $-\log(1-\eta\lambda)=1.0000500\times10^{-4}$ & $1.0000500\times10^{-4}$ \\
\bottomrule
\end{tabular}
\end{table}

\subsection{A Nonlinear Benchmark: Modular Addition}
\label{sec:modular}

We now leave the solvable model. The task is addition modulo $p=17$:
inputs are concatenated one-hot pairs ($n_0=2p$), targets one-hot, with $289$ pairs in total, of which a fraction $0.55$ ($158$ pairs) are used for training; the model is a two-layer MLP of width $128$ with quadratic
activation and squared loss in the Gromov--Omnigrok style
\citep{gromov2023grokking,liu2022omnigrok}, trained in full batch by
heavy ball with coupled $L_2$
($\beta=0.9$, $\eta=0.2$, $\lambda=3\times10^{-4}$), from a
large-norm initialization ($5\times$ standard scaling) that puts
training in the memorization-first regime. All results below are
single-seed and single-architecture.

\paragraph{Delayed generalization.}
Training accuracy reaches $1$ at $k\approx4\times10^{2}$; test
accuracy passes $0.5$ at $k\approx7\times10^{3}$ and $0.9$ at
$k\approx1.05\times10^{4}$, a delay ratio of $\approx26$, and ends
near $0.96$--$0.99$ (Figure~\ref{fig:modular}(a)).

\paragraph{Kernel mismatch.}
At a post-interpolation, pre-grokking checkpoint ($k=2000$) we compute
the exact training Jacobian ($2686\times6528$) and a held-out
Jacobian on $120$ unseen input pairs ($2040\times6528$). The training
Jacobian has full row rank ($2686$, stable under singular-value
cutoffs $10^{-3}$--$10^{-5}$), so
$\dim\widehat{\Ncal}_N=3842$; the held-out Jacobian restricted to
$\widehat{\Ncal}_N$ has full rank $2040$ at every cutoff. The
population-visible part of the empirical null space has maximal rank at this holdout size (2040 of 2040).

\paragraph{Transfer of the parameter dependence.}
Figure~\ref{fig:modular}(c) tests this dependence outside the solvable
model using a full factorial sweep over
$\eta\in\{0.2,0.5,1.0\}$,
$\lambda\in\{1,2,3\}\times10^{-4}$, and
$\beta\in\{0.8,0.9,0.95\}$, for a total of $27$ hyperparameter
configurations. All configurations use the same training split and
initialization to isolate the dependence on $(\eta,\lambda,\beta)$.
Across the runs that reach test accuracy $0.9$ within the prescribed
training horizon, the iteration at which this threshold is first
reached scales with $(1-\beta)/(\eta\lambda)$, with a fitted log--log
slope of $1.01$.

On the tail (Figure~\ref{fig:modular}(b)), the test-loss excess
relative to its late plateau decays exponentially at
$4.5\times10^{-4}$ per step, compared with the parameter-free
prediction $-\log\rho_0=6.0\times10^{-4}$, corresponding to agreement
within $25\%$. In contrast, the population-visible empirical-null
component measured using projectors frozen at $k=2000$ decays several
times more slowly. This discrepancy indicates substantial rotation of
the Jacobian during grokking, consistent with ongoing feature learning,
and shows that the frozen-projector description of
Section~\ref{sec:nonlinear} degrades over long time windows, as allowed
by its remainder terms.

\begin{figure}[t]
\centering
\includegraphics[width=\textwidth]{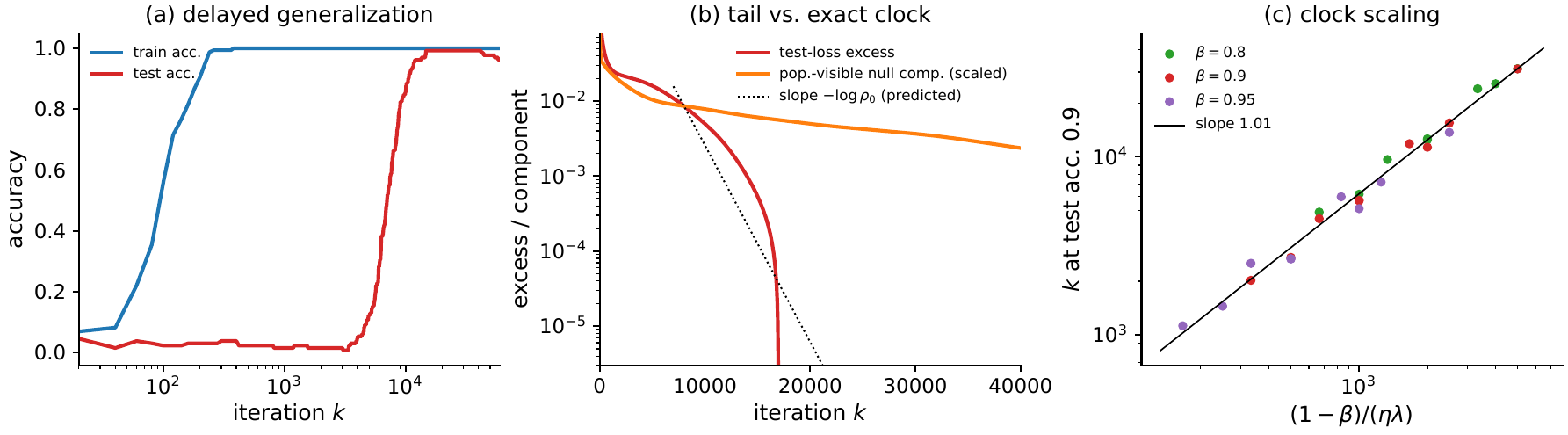}
\caption{Modular-addition benchmark
($p=17$, width $128$, quadratic activation, squared loss, heavy ball
with coupled $L_2$; single seed).
(a) Genuine delayed generalization at
$\eta=0.2$, $\lambda=3\times10^{-4}$, $\beta=0.9$.
(b) Test-loss excess relative to its late plateau, the
population-visible null component with projectors frozen at
$k=2000$ (scaled), and the parameter-free slope $-\log\rho_0$.
(c) Iteration at test accuracy $0.9$ against
$(1-\beta)/(\eta\lambda)$ across grokking configurations;
the fitted log--log slope is $1.01$.}
\label{fig:modular}
\end{figure}

\subsection{Scaling the Protocol}
\label{sec:protocol}

What the benchmark leaves open, we state as a protocol for
larger studies using the accompanying experiment package:
(i) multiple seeds with confidence intervals for grokking times,
fitted rates, and subspace dimensions, plus sensitivity of
$\widehat\Pg$ to SVD cutoff and holdout size (three cutoffs are
already reported above); (ii) interventions at depth: rescaling
$\widehat\Pg$-components and the momentum buffer
consistently with the slow manifold (rescaling $\theta$ alone excites the fast null root $\rho_1$), verifying that training predictions move only at
second order; (iii) gauge control along $\widehat\Pzero$ versus
$\widehat\Pg$; (iv) optimizer ablation (SGD+$L_2$, heavy
ball+$L_2$, SGDW as defined in Remark~\ref{rem:coupled-decoupled},
AdamW), probing the $(1-\beta)$ ratio of
Proposition~\ref{prop:decoupled}; (v) an $N$-sweep separating the
rate (predicted $N$-independent) from the $N$-dependent subspace
dimensions, amplitudes, and compatibility; and (vi) a margin
bridge from the loss tail to the accuracy jump.

\section{Related Work and Limitations}
\label{sec:related-work}

\paragraph{Grokking and delayed generalization.}
Grokking was introduced by \citet{power2022grokking} as a delayed
transition from memorization to generalization after training accuracy
has saturated on small algorithmic datasets.
\citet{liu2022effective} interpret the phenomenon through an effective
theory with memorization- and comprehension-dominated phases, whereas
\citet{liu2022omnigrok} attribute the transition to implicit
regularization toward lower-norm solutions.
\citet{junior2025grokking} showed that delayed generalization also
arises under sparsity- and rank-promoting regularizers, indicating that
Euclidean norm minimization is not the only relevant inductive bias.

Closest to us are the dynamical analyses.
\citet{boursier2025theoretical} established the $\frac{1}{\lambda}$ timescale of the post-interpolation Riemannian norm-minimizing flow; related quantitative delay laws appear in simplified settings \citep{xu2026grok,truong2026norm}. The
theorem-level differences of the present work are threefold
(Table~\ref{tab:comparison}). First, those analyses characterize
motion along the interpolation manifold; we decompose its
tangent space as $\Scal\oplus\Ncal_\pop$ and prove that only the
$\Scal$-component enters the slow empirical-null part of the
population-risk tail (Theorem~\ref{thm:population-tail}), including
the dichotomy $\rho_0^{2k}$ versus $\rho_0^{k}$ governed by an
explicit compatibility condition. Second, our rates are exact roots
of the discrete optimizer recurrence rather than continuous-time or
overdamped approximations, which yields the iteration-scale law
\cref{eq:kgrok-headline} and the coupled/decoupled optimizer
distinction (Proposition~\ref{prop:decoupled}) invisible at the level
of $\tau\propto1/\lambda$. Third, the intervention predictions
(P2--P3) are causal statements about the grokking subspace rather
than scaling laws; we are not aware of counterparts in prior delay
analyses.

\begin{table}[t]
\centering
\caption{Coarse theorem-level comparison with the closest dynamical
analyses of delayed generalization (as we read them; entries are
necessarily summary judgments).}
\label{tab:comparison}
\small
\begin{tabular}{|p{3.1cm}|p{3.4cm}|p{3.2cm}|p{3.4cm}|}
\hline
 & slow object & rate resolution & optimizer / causal \\
\hline
slow-manifold flow \citep{boursier2025theoretical} &
Riemannian norm flow on the interpolation manifold &
timescale $\propto1/\lambda$, continuous time &
not resolved at iteration scale; no interventions \\
\hline
delay laws \citep{xu2026grok,truong2026norm} &
delay in simplified/ridge settings &
quantitative in $\lambda$ &
not optimizer-resolved \\
\hline
Our work &
quotient $\Ker G_N/\Ker G_\pop$; only $\Scal$ in the slow tail &
exact discrete roots $\rho_0,\rho_1$; $\rho_0^{k}$ vs.\
$\rho_0^{2k}$ dichotomy &
coupled vs.\ decoupled factor $(1-\beta)$; interventions
(P2--P3) \\
\hline
\end{tabular}
\end{table}

\paragraph{Feature learning and mechanistic accounts.}
Mechanistic studies explain grokking through the gradual emergence of
structured internal representations.
\citet{nanda2023progress} showed that modular-addition transformers
develop Fourier-based circuits before the abrupt improvement in test
accuracy; \citet{gromov2023grokking} constructed an interpretable
two-layer model that groks modular arithmetic without explicit
regularization; \citet{mohamadi2024why} argued that early kernel-like
dynamics are insufficient for efficient generalization, whereas later
feature learning produces bounded-norm solutions with improved sample
complexity. For image classification, \citet{humayun2024deep} related
delayed generalization to changes in local input-space complexity.
Other theoretical work locates grokking in the passage between training regimes: \citet{lyu2024dichotomy} proved that distinct
implicit biases govern the two phases, \citet{kumar2024grokking}
interpreted the transition as a shift from lazy to rich training, and
\citet{zheng2024delays} connect delayed generalization with delayed
changes in representation geometry. These studies focus on how
representations evolve; our emphasis is on the parameter-space
directions that continue to evolve after interpolation, whatever
representation they encode.

\paragraph{Flat directions and separated timescales.}
\citet{beck2025edge} studied grokking near the threshold of linear
separability, where nearly flat directions produce critical slowing
down. Our decomposition sharpens ``flat direction'' into three
inequivalent objects (transverse, population-active null, and
population-null) with different clocks and different observability
(Figure~\ref{fig:intervention}(b)).

\paragraph{Symmetries and optimization geometry.}
Continuous parameter symmetries are common in homogeneous and deep
linear networks; rescaling symmetries lead to non-identifiability,
balanced factorizations, and structured implicit biases
\citep{saxe2014exact,du2018algorithmic}.
\citet{tanaka2021noether} and
\citet{marcotte2023abide,marcotte2024momentum} derived Noether-type
conservation laws for gradient flows with and without momentum; we
add the \emph{softly broken} counterpart with weight decay as the
explicit breaking term, and center it on the translation charges of
the empirical null space, whose breaking law is the grokking
coordinate's equation of motion. Neural tangent and Fisher metrics
provide pullback geometries measuring changes in predictions rather
than parameters \citep{amari1998natural,jacot2018neural}; we use them
as observability Grams, not as kinetic masses, precisely because
$G_N$ is singular in the overparameterized regime
(Section~\ref{sec:kernel-geometry}).

\paragraph{Limitations.}
The exact theory requires linearity in parameters and squared loss;
the nonlinear extension is local, continuous-time only, carries the
remainder assumptions and error floor of
Section~\ref{sec:nonlinear}, and the modular benchmark already shows
its frozen-projector description degrading quantitatively under
feature learning. The grokking subspace $\Scal$ is the
metric-dependent representative of an intrinsic quotient
(Section~\ref{sec:kernel-geometry}); neither $\Scal$ nor $\Pg$ is
reparameterization-invariant. The analysis is full-batch: stochastic
gradients add noise on $\Ncal_N$ that competes with the weight-decay
drift and can change the clock. Dataset size enters only through
subspace dimensions, amplitudes, and compatibility; the theory
makes no quantitative $N$-law for the delay. The population Gram
$G_\pop$ is an idealized object; any practical test uses a finite
held-out estimate with its own sampling error. The theory predicts
the decay of a signed population-risk excess relative to
$\theta_\eq$; converting a loss tail into the sharp jump of test
accuracy requires a margin argument that we do not supply.
The synthetic experiments are theorem verifications, and the
nonlinear benchmark is single-seed and minimal. Finally, adaptive and
preconditioned optimizers are outside the proven scope
(Appendix~\ref{app:natural-gradient}).

\section{Conclusion}
\label{sec:conclusion}
We proposed an exactly solvable mechanism for a post-interpolation
component of delayed population-risk relaxation that can contribute
to grokking. In the solvable regime --- linear-in-parameter models,
squared loss, full-batch heavy ball --- the empirical null space
carries a Euclidean algebra of prediction-preserving symmetries;
weight decay supplies the restoring force that softly breaks the
translation part, damping and inertia convert that breaking scale
into a rate, and the resulting charge-balance law is the equation of
motion of the slow coordinates. The mismatch between the empirical
and population kernels splits those coordinates into population-null
directions and the grokking subspace $\Scal$ (the metric-chosen
representative of an intrinsic quotient), and only the latter feeds
the slow component of the asymptotic tail of the signed excess
relative to the regularized equilibrium, with exact discrete rates:
the training loss runs on the fast clocks $(\beta^{\,k/2 }$
envelopes, $\beta^k$ rotation charges$)$ and is provably blind to
the slow root $\rho_0$, while the signed population excess carries
$\rho_0^{k}$ or $\rho_0^{2k}$ according to an explicit compatibility
condition. We verified the resulting predictions without fitted parameters in the solvable model: the $(1-\beta)/(\eta\lambda)$ iteration law, the logarithmic intervention shift, the train/test signature of $\Scal$ versus $\Ncal_\pop$, and the coupled/decoupled optimizer distinction. The iteration law was reproduced (log–log slope $1.01$) in a modular-addition benchmark that exhibits delayed generalization; validation at depth, across seeds, and for the remaining predictions is open.

The priorities for extending the theory are concrete. First,
stochastic gradients: mini-batch noise projects onto $\Ncal_N$ and
competes with the $\lambda$-drift, plausibly renormalizing the clock;
a stochastic soft Noether law is the natural tool. Second,
finite-sample estimation of $G_\pop$ and of the projectors
$\widehat \Pg$, with error bounds that make the protocol of
Section~\ref{sec:protocol} quantitative. Third, discrete momentum
beyond heavy ball (Nesterov and adaptive preconditioning), where
Proposition~\ref{prop:decoupled} already shows the clock is
optimizer-dependent. Fourth, the margin-to-accuracy conversion that
turns a population-loss tail into the abrupt accuracy jump that gave
grokking its name.

\subsubsection*{Acknowledgments}
Taeyoung Kim is supported by a KIAS Individual Grant (AP102201) at
Korea Institute for Advanced Study and supported by the Center for
Advanced Computation at Korea Institute for Advanced Study.


\bibliography{references}


\appendix

\section{Soft Noether Law for General Kinetic Metrics}
\label{app:noether}

This appendix derives the charge-drift law used in
Section~\ref{sec:symmetries}, in the generality of a
$\theta$-dependent positive-definite kinetic metric $M(\theta)$, and
records the distinction between the two inequivalent damping conventions that
coincide only for scalar $M$.

\subsection{Two Damping Conventions}
\label{app:damping-conventions}

For a positive-definite mass tensor $M(\theta)$ define
\begin{equation}
    \ham_M(\theta,p)
    =
    \tfrac12\,p^\top M(\theta)^{-1}p+U(\theta),
    \qquad
    p=M(\theta)\dot\theta .
\label{eq:app-general-hamiltonian}
\end{equation}
Two dissipative extensions of the canonical flow appear in the
literature and describe \emph{different} dynamics for non-scalar $M$:
\begin{align}
    \textbf{(momentum damping)}\quad
    &\dot p=-\partial_\theta\ham_M-\frac{\gamma}{m_0}\,p
    &&\Longrightarrow\quad
    M\ddot\theta+\frac{\gamma}{m_0}M\dot\theta+\nabla U
    =F_M(\theta,\dot\theta),
    \label{eq:app-momentum-damping}\\
    \textbf{(velocity damping)}\quad
    &\dot p=-\partial_\theta\ham_M-\gamma\,\dot\theta
    &&\Longrightarrow\quad
    M\ddot\theta+\gamma\dot\theta+\nabla U
    =F_M(\theta,\dot\theta),
    \label{eq:app-velocity-damping}
\end{align}
where $F_M$ collects the Christoffel-type terms generated by
$\partial_\theta M$ (they vanish for constant $M$) and $m_0$ is a
reference mass scale. For $M=mI$ (with $m_0=m$) the two conventions
coincide and reduce to \cref{eq:damped-particle}; this is the only
case used in the main text. For general $M$ they differ by the
operator $M/m_0$ multiplying the friction, they have different
Lyapunov structures, and a Noether analysis must fix one convention
throughout. Statements that mix
\cref{eq:app-momentum-damping} for the charge law with
\cref{eq:app-velocity-damping} for the overdamped limit are not
consistent for $M\neq mI$.

\subsection{Charge Drift under Momentum Damping}

We work with \cref{eq:app-momentum-damping}; the modification for
\cref{eq:app-velocity-damping} is indicated at the end. Let
$Q_X=\ip{p}{X(\theta)}=p_iX^i(\theta)$.

\paragraph{Step 1: the charge splits into a Poisson bracket and a
drift.}
Differentiating along the flow,
\begin{align}
    \frac{dQ_X}{dt}
    &=
    \dot p_i\,X^i+p_i\,(\partial_jX^i)\,\dot\theta^j
    =
    \underbrace{p_i(\partial_jX^i)\dot\theta^j-X^i\partial_i\ham_M}_{\{Q_X,\ham_M\}}
    \;-\;
    \frac{\gamma}{m_0}\,Q_X .
\label{eq:app-charge-split}
\end{align}

\paragraph{Step 2: the position derivative of $\ham_M$.}
Using $\partial_k(M^{-1})=-M^{-1}(\partial_kM)M^{-1}$ and
$M^{-1}p=\dot\theta$,
\begin{equation}
    \partial_k\ham_M
    =
    -\tfrac12\,\dot\theta^\top(\partial_kM)\dot\theta+\partial_kU .
\label{eq:app-dH}
\end{equation}

\paragraph{Step 3: assembling the Lie derivative.}
Substituting \cref{eq:app-dH} and $p_i=M_{ij}\dot\theta^j$ into
\cref{eq:app-charge-split} and using
$(\Lie_XM)_{ij}=X^k\partial_kM_{ij}+M_{kj}\partial_iX^k
+M_{ik}\partial_jX^k$,
\begin{equation}
    \{Q_X,\ham_M\}
    =
    \tfrac12\,\dot\theta^\top(\Lie_XM)\dot\theta
    -\ip{X}{\nabla U}.
\label{eq:app-bracket}
\end{equation}

\paragraph{Step 4: friction and weight decay.}
With $U=\loss_N+\frac\lambda2\norm\theta^2$,
\begin{equation}
    \frac{dQ_X}{dt}
    =
    \tfrac12\,\dot\theta^\top(\Lie_XM)\dot\theta
    -\ip{X}{\nabla\loss_N}
    -\frac{\gamma}{m_0}\,Q_X
    -\lambda\ip{\theta}{X}.
\label{eq:app-full-drift}
\end{equation}
Imposing the Killing condition $\Lie_XM=0$ and loss invariance
$\ip{X}{\nabla\loss_N}=0$ leaves the soft Noether law
\[
    \frac{dQ_X}{dt}
    =
    -\frac{\gamma}{m_0}\,Q_X-\lambda\ip{\theta}{X},
\]
which specializes to
\cref{eq:soft-noether-translation} for $M=mI$, $X\equiv b$. Under
velocity damping \cref{eq:app-velocity-damping} the drift term is
instead $-\gamma\ip{\dot\theta}{X}=-\gamma\,\ip{M^{-1}p}{X}$, which is
not proportional to $Q_X$ unless $M$ is scalar --- one more reason the
main text fixes $M=mI$.

\subsection{Affine Generators for a General Constant Metric}

\begin{theorem}[General-metric affine generators]
\label{thm:general-generators}
Let $M\succeq0$ be a constant metric, $M^{1/2}$ its PSD square root,
$(M^{1/2})^{\dagger}$ the Moore--Penrose pseudoinverse, and
$P:=M^{1/2}(M^{1/2})^{\dagger}$ the orthogonal projection onto
$\Ran(M)$. Consider $X_{b,A}(\theta)=b+A(\theta-\theta_\eq)$ for the
linear model. Then $X_{b,A}$ is a Killing field of $M$ and preserves
the empirical predictions iff
\begin{equation}
    \mathsf J_Nb=0,\qquad
    \mathsf J_NA=0,\qquad
    A^\top M+MA=0 .
\label{eq:app-generator-conditions}
\end{equation}
Translations require no metric condition ($\Lie_bM=0$ automatically
for constant $M$). For the rotation part, let
$U_0\in\R^{P\times r}$ have orthonormal columns spanning
$\Ker\bigl(\mathsf J_N(M^{1/2})^{\dagger}\bigr)\cap\Ran(M)$. A matrix
$A$ solves \cref{eq:app-generator-conditions} and satisfies the
compatibility conditions $\Ran(A)\subseteq\Ran(M)$ and
$\Ker(M)\subseteq\Ker(A)$ (equivalently $A=PAP$) if and only if
\begin{equation}
    A
    =
    (M^{1/2})^{\dagger}\,U_0\,\Omega\,U_0^\top M^{1/2},
    \qquad
    \Omega^\top=-\Omega .
\label{eq:app-explicit-generator}
\end{equation}
When $M\succ0$ the compatibility conditions are vacuous and
\cref{eq:app-explicit-generator} exhausts all solutions; the symmetry
algebra is $\Ncal_N\rtimes\so(r)$ with $r=\dim\Ker(\mathsf
J_N(M^{1/2})^{\dagger})$.
\end{theorem}

\begin{proof}
The prediction condition and the Killing computation are as in
Theorem~\ref{thm:affine-generators}. For the rotation part, restrict
to $A=PAP$ and set $B:=M^{1/2}A(M^{1/2})^{\dagger}$; the map
$A\mapsto B$ is a bijection onto operators with
$\Ran(B)\subseteq\Ran(M)$, conjugation turns the metric condition
into $B^\top=-B$ and the prediction condition into
$\Ran(B)\subseteq\Span(U_0)$; skew-symmetry forces
$B=U_0\Omega U_0^\top$, and transforming back gives
\cref{eq:app-explicit-generator}. Conversely any such $A$ solves
\cref{eq:app-generator-conditions}.
\end{proof}

\begin{remark}
\label{rem:app-kernel-gauge}
When $M$ is singular, additional solutions of
\cref{eq:app-generator-conditions} act into or out of $\Ker(M)$; they
carry zero kinetic energy and vanishing charge, hence are dynamically
trivial for the charge analysis. Using a \emph{singular} $M$ (such as
$G_N$ itself) as a kinetic mass is, however, not merely a matter of
adding such generators: the Legendre transform $p=M\dot\theta$ fails
to be invertible, and the Hamiltonian flow is not defined without a
regularization $M_\epsilon=M+\epsilon I$; see
Appendix~\ref{app:natural-gradient}.
\end{remark}

\section{Deferred Proofs}
\label{app:proofs}

\subsection{Proof of Theorem~\ref{thm:discrete-spectrum}}
\label{app:proofs-discrete}

The recurrence \cref{eq:exact-recurrence} preserves the eigenspaces of
$G_N$, giving \cref{eq:modal-recurrence}; the general solution of a
scalar three-term recursion with distinct roots is
$c^+\rho^{+k}+c^-\rho^{-k}$, with the confluent form at a double
root.

\emph{(i) Stability.} The roots of
$\rho^2-a\rho+\beta=0$ satisfy $\rho^+\rho^-=\beta<1$ and
$\rho^++\rho^-=a$. If complex, $|\rho^\pm|=\sqrt\beta<1$ always. If
real, both roots lie in $(-1,1)$ iff the polynomial
$p(\rho)=\rho^2-a\rho+\beta$ satisfies $p(1)>0$ and $p(-1)>0$, i.e.\
$1-a+\beta>0$ and $1+a+\beta>0$, i.e.\
$0<\eta(h_j+\lambda)<2(1+\beta)$.

\emph{(ii) Null modes.} For $h=0$, $a_0=1+\beta-\eta\lambda$. The
discriminant $a_0^2-4\beta$ is positive iff
$a_0>2\sqrt\beta$, i.e.\ $\eta\lambda<(1-\sqrt\beta)^2$; both roots
are then positive (product $\beta>0$, sum $a_0>0$), with
$\rho_0\rho_1=\beta$ and $\rho_0>\sqrt\beta>\rho_1$;
$\rho_1=\beta/\rho_0\ge\beta$ follows from $\rho_0\le1$, and
$\rho_0<1$ from $p(1)=\eta\lambda>0$. Expanding
$\sqrt{a_0^2-4\beta}
=(1-\beta)\sqrt{1-2(1+\beta)\eta\lambda/(1-\beta)^2+O((\eta\lambda)^2)}$
gives
$\rho_0=1-\eta\lambda/(1-\beta)+O\bigl((\eta\lambda)^2/(1-\beta)^3\bigr)$.

\emph{(iii) Transverse modes.} On the real branch with $a>0$
($\eta(h+\lambda)\le(1-\sqrt\beta)^2$), the larger root
$\rho^+(a)=\frac12(a+\sqrt{a^2-4\beta})$ is strictly increasing in
$a$, and $a_j<a_0$ for $h_j>0$, so $\rho_j^+<\rho_0$; the smaller
root satisfies $\rho_j^-=\beta/\rho_j^+<\rho_j^+$. On the complex
branch, $|\rho_j^\pm|=\sqrt{\rho_j^+\rho_j^-}=\sqrt\beta<\rho_0$ by
(ii). The excluded sliver
$\eta(h+\lambda)\in\bigl((1+\sqrt\beta)^2,2(1+\beta)\bigr)$ has real
negative roots whose modulus approaches $1$ at the stability
boundary and can exceed $\rho_0$; the no-sign-flip condition removes
it. \hfill$\blacksquare$

\subsection{Proof of Corollary~\ref{cor:iteration-clock}}

On $\Ncal_N$, $\PN\xi_k=\rho_0^kw_0+\rho_1^kw_1$, so
$\Pg\xi_k=\rho_0^k\bigl(\Pg w_0+(\rho_1/\rho_0)^k\Pg w_1\bigr)$.
For $k\ge k_1:=\log\bigl(2\norm{\Pg w_1}/\norm{\Pg w_0}\bigr)
/\log(\rho_0/\rho_1)$ the bracket lies within a factor
$[\tfrac12,\tfrac32]$ of $\norm{\Pg w_0}$, and solving
$\rho_0^k\,\norm{\Pg w_0}=\varepsilon$ for $k$ gives
\cref{eq:kgrok-exact} with an $O(1)$ error containing $k_1$ and the
factor bounds; \cref{eq:kgrok-headline} follows from
\cref{eq:slow-root-expansion}. \hfill$\blacksquare$

\subsection{Proof of Theorem~\ref{thm:population-tail}}
\label{app:proofs-tail}

Because $\loss_\pop$ is exactly quadratic,
\[
    \loss_\pop(\theta_k)-\loss_\pop(\theta_\eq)
    =
    \ip{\nabla_\pop}{\xi_k}
    +\tfrac12\,\xi_k^\top G_\pop\,\xi_k .
\]
Decompose $\xi_k=\rho_0^kw_0+\rho_1^kw_1+\xi_k^\perp$, where
$\xi_k^\perp=\sum_{j:h_j>0}(c_j^+\rho_j^{+k}+c_j^-\rho_j^{-k})u_j$
satisfies $\norm{\xi_k^\perp}\le C_\perp\bar\rho^{\,k}$ with
$\bar\rho=\max_{j:h_j>0}|\rho_j^\pm|$ (double roots contribute
$k\rho^k=O((\rho+\epsilon)^k)$ absorbed by enlarging $\bar\rho$
infinitesimally within Assumption~\ref{ass:tail-separation}).

\emph{Linear term.} $\ip{\nabla_\pop}{w_0}\rho_0^k=A\rho_0^k$ using
$\nabla_\pop\perp\Ncal_\pop$ and $w_0\in\Ncal_N=\Scal\oplus\Ncal_\pop$,
so only $s_0=\Pg w_0$ contributes. The remaining pieces are
bounded by
$\norm{\nabla_\pop}(\norm{w_1}\rho_1^k+C_\perp\bar\rho^{\,k})
\le C_1\varrho^{\,k}$.

\emph{Quadratic term.} Expand in the three blocks. The slow--slow
piece is $\tfrac12w_0^\top G_\pop w_0\,\rho_0^{2k}=B\rho_0^{2k}$,
using $G_\pop \Pzero w_0=0$ and $\Pg w_0\perp\Ncal_\pop$ to reduce
$w_0^\top G_\pop w_0=s_0^\top G_\pop s_0$. Every other product
carries a factor pair among
$\{\rho_0\rho_1,\rho_0\bar\rho,\rho_1^2,\rho_1\bar\rho,\bar\rho^2\}$;
each is bounded by $\varrho$ (using $\rho_0\rho_1=\beta\le\rho_1$,
$\rho_0<1$), giving a total bound $C_2\varrho^{\,k}$. Set
$C=C_1+C_2$.

\emph{Training blindness.} $\loss_N$ is quadratic with
$\nabla\loss_N(\theta_\eq)=\nabla U(\theta_\eq)-\lambda\theta_\eq
=-\lambda\theta_\eq\in\Ncal_N^\perp$ and Hessian $G_N$, both of which
annihilate $\Ncal_N$: neither the linear nor the quadratic term of
$\loss_N(\theta_k)-\loss_N(\theta_\eq)$ contains
$\rho_0^k$ or $\rho_1^k$ factors paired with nonzero coefficients,
except through $\xi^\perp_k$; the bound
\cref{eq:training-blind} follows.

\emph{Positivity.} For $s\in\Scal\setminus\{0\}$:
$s\perp\Ker G_\pop$ and $G_\pop\succeq0$ imply $s^\top G_\pop s>0$.
\hfill$\blacksquare$

\subsection{Proof Sketch of Proposition~\ref{prop:rate-bracket}}

Write the continuous dynamics for
$\zeta:=\Pg\xi$ (projector frozen at $\theta_\star$):
$m\ddot\zeta+\gamma\dot\zeta+\lambda\zeta
=-\Pg\nabla\loss_N(\theta)-\lambda \Pg\theta_\star
+E_{\mathrm{rot}}(t)$. By Assumption~\ref{ass:local-linear},
$\norm{\PN\nabla\loss_N(\theta)}\le C_N'\norm\xi^2$ on
$\mathcal U$ (the first-order term lies in
$\Ran(\mathsf J_N^\top)=\Ncal_N^\perp$), and
$\norm{\PN\theta_\star}\le c_0\lambda$; the projector-drift term
$E_{\mathrm{rot}}$ is bounded, via the Davis--Kahan estimate of
Assumption~\ref{ass:gap}, by $(2L_J/\sigma_0)\,\delta$ times the
transverse force, and enters $C_\delta$. In the overdamped regime the
slaved first-order dynamics is
$\gamma\dot\zeta=-\lambda\zeta+e(t)$ with
$\norm{e(t)}\le F+\gamma C_\delta\norm\zeta$,
$F=O(C_N'\delta^2+c_0\lambda^2)$. Since
$\bigl|\frac{d}{dt}\norm\zeta\bigr|\le\norm{\dot\zeta}$, this yields
the two differential inequalities
\[
    \frac{d}{dt}\norm\zeta
    \le
    -(r_{\mathrm{slow}}-C_\delta)\norm\zeta+F/\gamma,
    \qquad
    \frac{d}{dt}\norm\zeta
    \ge
    -(r_{\mathrm{slow}}+C_\delta)\norm\zeta-F/\gamma,
\]
which integrate by variation of constants to
\cref{eq:rate-upper}--\cref{eq:rate-lower}. The additive terms are
the integrated forcing: they cannot be removed from a norm bound,
because direction rotation and cancellation prevent a pointwise
multiplicative lower bound once $\norm\zeta$ is comparable to the
floor $F/(\gamma r_{\mathrm{slow}})$. Above a fixed multiple of the
floor, dividing \cref{eq:rate-upper}--\cref{eq:rate-lower} by
$\norm{\zeta(t_0)}$ gives the conditional bracket
\cref{eq:rate-bracket}; the population statement follows by
inserting the bounds into \cref{eq:pop-expansion}.
\hfill$\blacksquare$

\section{Outlook: Preconditioned Dynamics}
\label{app:natural-gradient}

For the regularized NTK metric $M_\epsilon=G_N+\epsilon I$ with
$\epsilon\to0^+$, two heuristic expectations follow from the modal
picture. \emph{Transverse whitening}: in either damping convention of
Appendix~\ref{app:damping-conventions}, the overdamped transverse flow
is preconditioned by $M_\epsilon^{-1}$ and the transverse rates
whiten, $(h_j+\lambda)/(\gamma(h_j+\epsilon))\approx1/\gamma$,
independent of the curvature spectrum. \emph{Convention-dependent null
clock}: along $\Scal\subseteq\Ker G_N$ the effective inertia is
$\epsilon$, and the null-mode oscillator is
$\epsilon\ddot q+\gamma_{\mathrm{eff}}\dot q+\lambda q=0$ with
$\gamma_{\mathrm{eff}}=\gamma$ under velocity damping but
$\gamma_{\mathrm{eff}}=(\gamma/m_0)\epsilon$ under momentum damping.
The slow rate is then $\approx\lambda/\gamma$ (unchanged) under velocity damping, while under momentum damping it is
$\approx\lambda m_0/(\gamma\epsilon)$, growing as $\epsilon\to0$ and
saturating at $\gamma/(2m_0)$ once $\epsilon<4\lambda m_0^2/\gamma^2$. Whether a loss-adapted metric removes the delay depends
entirely on which dissipation model the optimizer realizes.

Three reasons this stays an outlook. First, $G_N$ is singular, so the
Legendre transform is undefined at $\epsilon=0$ and any statement is
about a limit $\epsilon\to0^+$ whose uniformity in $t$ is not
established. Second, for $M\neq mI$ the two damping conventions of
Appendix~\ref{app:damping-conventions} differ, and mapping
an actual optimizer (natural gradient with damping, Adam-type
preconditioning) onto one of them is itself a modeling decision.
Third, the momentum-damping version predicts that preconditioned
optimizers largely eliminate the delay at fixed $\lambda$ --- a testable prediction that should be settled empirically before any theorem is attempted.

\end{document}